\documentclass[12pt,a4paper]{article}

\usepackage[T1]{fontenc}
\usepackage[utf8]{inputenc}
\usepackage{authblk}

\usepackage{amsfonts}
\usepackage{graphicx}
\usepackage{graphics}
\usepackage{subfigure}
\usepackage[colorlinks,linkcolor=red,anchorcolor=blue,citecolor=blue]{hyperref}
\usepackage{color}
\usepackage{mathrsfs,amsthm,amsmath,amssymb}
\usepackage{enumerate}
\usepackage{array}
\usepackage{appendix}
\usepackage{framed}
\usepackage{accents}
\usepackage{tcolorbox}
\usepackage{bbm}
\usepackage{geometry}
\usepackage{hyperref}
\usepackage{multirow}
\usepackage{makecell}
\usepackage{placeins}
\usepackage{cleveref}
\usepackage{bm}
\usepackage{algorithm}
\usepackage{algorithmic}
\usepackage{float}
\usepackage{epstopdf}
\usepackage{setspace}
\usepackage{mathtools}
\usepackage{booktabs}
\usepackage{array}
\usepackage{caption}
\theoremstyle{plain}
\newtheorem{theorem}{Theorem}[section]
\newtheorem{corollary}{Corollary}[section]
\newtheorem{lemma}{Lemma}[section]
\newtheorem{proposition}{Proposition}[section]

\theoremstyle{definition}
\newtheorem{definition}{Definition}

\theoremstyle{remark}
\newtheorem{remark}{Remark}

\allowdisplaybreaks

\title{Exploratory Utility Maximization Problem with Tsallis Entropy\thanks{The data that support the findings of this study are available from the corresponding author upon reasonable request.
The authors declare that they have no known competing financial interests or personal relationships that could have appeared to influence the work reported in this paper.}  }

\author{
Ziyi Chen
\thanks{Department of Mathematics,
Southern University of Science and Technology, Shenzhen, China.  E-mail: 12231268@mail.sustech.edu.cn.
}
\and
Jia-Wen Gu
\thanks{
Corresponding author. Department of Mathematics,
Southern University of Science and Technology, Shenzhen, China.  E-mail: gujw@sustech.edu.cn. Supported in part by the Key Project of National Natural Science Foundation of China 72432005, Guangdong Basic and Applied Basic Research
Foundation 2023A1515030197, Shenzhen Humanities \& Social Sciences Key Research Bases.}
}

\date{}

\begin{document}
	\maketitle
	\begin{abstract}
		We study expected utility maximization problem with constant relative risk aversion utility function in a complete market under the reinforcement learning framework. To induce exploration, we introduce the Tsallis entropy regularizer, which generalizes the commonly used Shannon entropy. Unlike the classical Merton's problem, which is always well-posed and admits closed-form solutions, we find that the utility maximization exploratory problem is ill-posed in certain cases, due to over-exploration. With a carefully selected primary temperature function, we investigate two specific examples, for which we fully characterize their well-posedness and provide semi-closed-form solutions. 
        It is interesting to find that one example has the well-known Gaussian distribution as the optimal strategy, while the other features the rare Wigner semicircle distribution, which is equivalent to a scaled Beta distribution. The means of the two optimal exploratory policies coincide with that of the classical counterpart. In addition,  we examine the convergence of the value function and optimal exploratory strategy as the exploration vanishes. Finally, we design a reinforcement learning algorithm and conduct numerical experiments to demonstrate the advantages of reinforcement learning. 
	\end{abstract}
	
	
	\noindent\emph{Key words: Reinforcement learning; Utility maximization problem; Well-posedness; Tsallis entropy; Exploratory function}
	
	
	\section{Introduction}\label{sec:introduction}
	Merton \cite{merton1975optimum} is the first to consider financial markets in continuous time, studying the optimal investment strategy of a investor with constant relative risk aversion (CRRA) with the objective of maximizing terminal utility. In the absence of transaction costs and consumption, he derived the optimal strategy by solving the Hamilton-Jacobi-Bellman (HJB) equation, which leads to a constant ratio between risky and risk-free assets. This work provided a foundational framework for subsequent portfolio optimization studies. Subsequent variations of the model introduced factors such as recursive utility (e.g., Epstein and Zin \cite{epstein2013substitution}, Duffie and Epstein \cite{duffie1992stochastic}), transaction costs (e.g., Shreve and Soner \cite{shreve1994optimal}, Liu et al.\cite{liu2002optimal}, Dai and Yi \cite{dai2009finite}) and constraints on state or control (e.g., Elie et al. \cite{elie2008optimal}, Lim et al. \cite{lim2011optimal}). It is worth noting that in practical applications the model parameters need to be estimated and then plugged into the optimal control function. This represents a model-dependent classical control approach. The disadvantage of this approach is that the plug-in strategy is highly sensitive to the model parameters and may no longer be optimal as market conditions fluctuate.\\
	\indent To overcome the issue of parameter estimation, researchers have turned to reinforcement learning (RL). RL is a model-free machine learning approach that learns optimal strategies through interaction with the environment. It is based on a reward-feedback mechanism, where an agent interacts with the environment by taking actions, receiving rewards or penalties, and adjusting its behavior based on this feedback. The goal is to maximize the cumulative long-term reward through exploration and exploitation. For the discrete case, the RL algorithms are already well-developed (e.g., Sutton \cite{sutton2018reinforcement}, Szepesv{\'a}ri and Csaba \cite{szepesvari2022algorithms}). However, for the continuous case, particularly with stochastic dynamics, RL is still under development. Current research in continuous time can be divided into two categories: one focuses on theoretical and algorithmic studies, while the other applies RL to specific problems. On the theoretical side, Wang and Zhou \cite{wang2020reinforcement} are the first to propose an exploratory stochastic control framework in continuous time with RL. They randomize the classical control into a distributional strategy and introduce entropy regularization, obtaining the optimal feedback control distribution, which is Gaussian by solving the exploratory HJB equation. From the perspective of PDEs, Tang et al. \cite{tang2022exploratory} establish the well-posedness and regularity of the viscosity solution to the exploratory HJB equation, and show that the exploratory control problem converges to the classical stochastic control problem as exploration diminishes.  Jia and Zhou \cite{jia2022policy} propose a unified framework to study policy evaluation (PE), specifically to learn the value function given a policy. They present two methods for using martingale characterization to design PE algorithms. One method minimizes the so-called martingale loss function, while the other is based on a system of equations known as the ``martingale orthogonality conditions'' with test functions. Immediately afterward, Jia and Zhou \cite{jia2022policy_a,jia2023q} devise two actor-critic algorithms based on the policy gradient and q-learning.  
	On the application side, Wang and Zhou \cite{wang2020continuous} and Dai et al. \cite{dai2023learning} study the portfolio optimization problem under the MV criterion using reinforcement learning. Due to the time inconsistency arising from the MV problem,  Wang and Zhou \cite{wang2020continuous} solve for the pre-commitment distributional strategy, while Dai et al. \cite{dai2023learning} solve for the equilibrium distributional strategy. Wu and Li \cite{wu2024reinforcement} employ the proposed RL method to address the continuous-time mean-variance portfolio optimization problem in a regime-switching market. Dong \cite{dong2024randomized} and Dai et al. \cite{dai2024learning} introduce RL into the optimal stopping problem.\\
	\indent Classical expected utility maximization problems with CRRA utility function are always well-posed and admit explicit solutions. However, when exploration is introduced, utility maximization problems become more complex.  { Several studies have examined the problem of utility maximization under exploration. For instance, Jiang et al. \cite{jiang2022reinforcement} focus on  a log utility function and derive a closed-form solution. Meanwhile, Dai et al. \cite{dai2023learning} investigate the recursive entropy-regularized utility maximization problem, and Bo et al. \cite{bo2024continuous} study general reinforcement learning in jump-diffusion models (not necessarily LQ) by featuring q-learning under Tsallis entropy regularization. However, to the best of our knowledge, the well-posedness of the exploratory utility maximization problem has not yet been thoroughly studied.} In this paper, we will show that in some cases this problem is actually not well-posed, due to over-exploration. We are surprised to find that the choice of primary temperature function, the combination of market parameters, and the selection of the utility function all influence the well-posedness of the problem and attainability of (semi-)closed-form strategies.
    Our main contribution is to fully characterize the well-posedness of the exploratory utility maximization problem under Tsallis entropy and derive (semi-)closed-form solutions with a carefully selected primary temperature
function.\\
	\indent To encourage exploration, as Bo et al. \cite{bo2024continuous} and Donnelly et al. \cite{donnelly2024exploratory}, we introduce the Tsallis entropy (Tsallis \cite{tsallis1988possible}), which is a generalization of the Shannon entropy. Incorporating different regularization terms induces different optimal distribution strategies. For instance, Shannon entropy typically leads to a Gaussian distribution, while Choquet regularizers can generate a variety of widely used exploratory samplers, including exponential, uniform, and Gaussian distributions (Han et al.\cite{han2023choquet}, Guo et al. \cite{guo2023exploratory}). With a carefully selected primary temperature function, we provide two examples under the Tsallis entropy with semi-closed-form solutions. This temperature function ensures that the exploratory HJB equation is homogeneous, allowing the PDE to be reduced to an ODE through dimensionality reduction. We offer a detailed discussion of the existence and uniqueness of the ODE solution and characterize its properties. Notably, one example illustrates that the optimal distribution induced by the Tsallis entropy index being 3 is no longer Gaussian, but rather a semicircular distribution, which can be viewed as a scaled Beta distribution. In addition, we investigate the convergence of the value function and optimal exploratory strategy as the exploration vanishes. Finally, we adopt the actor-critic algorithm proposed by Jia and Zhou \cite{jia2022policy,jia2022policy_a} to carry out the numerical experiment and present substantial simulation results. \\
	\indent The remainder of the paper is organized as follows. In Section \ref{sec:models}, we present the model setup and discuss how to choose an appropriate primary temperature function. In Section \ref{sec:optimal_strategy}, we derive the optimal distributional strategy corresponding to Tsallis entropy by solving the exploratory HJB equation. In addition, we provide two specific examples with a semi-closed-form solution and also study the convergence properties of both the strategy and the value function. We introduce the RL algorithm design in Section \ref{sec:RL} and provide the corresponding numerical results in Section \ref{sec:simulation}. Finally, we conclude this paper in Section \ref{sec:conclusions}.
	
	\section{Problem Formulation}\label{sec:models}
	\subsection{Model Setup}
	Assume there are two assets in a market without transaction costs: one is a bond with risk-free interest rate $r$ and the other is a risky stock. The price process of the risky stock is governed by a geometric Brownian motion:
	\begin{equation}\label{eq:GBM}
		dS_t = \mu S_tdt+ \sigma S_t dB_t, ~ S_0=s_0>0,
	\end{equation}
	where $\mu \in R$ is mean, $\sigma>0$ is volatility, and $B_t$ is a standard one-dimensional Brownian motion defined on a filtered probability space $(\Omega, \mathcal{F},\{\mathcal{F}_t\}_{t \geq 0}, \mathbb P)$ that satisfies the usual conditions. 
	Denote the fraction of total wealth invested in the stock at time $t$ by $u_t$. Then the corresponding self-financing wealth process $W_t$ satisfies the following SDE:
	\begin{equation}\label{cl_wealth}
		dW_t = [r+(\mu-r)u_t]W_t dt +\sigma u_t W_t dB_t, ~W_0=w_0>0.
	\end{equation}
	Let $T$ be the investment horizon. In classical Merton's problem, investors are to maximize the expected utility of the terminal wealth $W_T$:
	\begin{equation}\label{pb:classical}
		\max\limits_{u\in \mathcal{U}(t, w)}\mathbb E [U_p(W_T) \mid W_t=w]
	\end{equation} 
	where $U_p(\cdot)$ is a utility function and $\mathcal{U}(t, w)$ is  admissible control set defined as 
	\begin{equation*}
		\mathcal{U}(t,w)=\Big\{u_s \ {\rm is} \ \mathcal{F}_s-{\rm progressively \  measurable}: \mathbb E \Big[\int_{t}^{T}(u_sW_s)^2ds \mid W_t=w \Big]<+\infty\Big\}.
	\end{equation*}
	In this paper, we focus on the following constant relative risk aversion (CRRA) utility function defined for  $w > 0$:
	\begin{equation*}
		U_p(w)\triangleq\begin{cases}\frac{w^p}{p},&\text{if }p<1, p\neq0,\\\log w,&\text{if }p=0,\end{cases}
	\end{equation*}
	where we set $U_p(0)\triangleq -\infty$ if $p \leq 0$.\\ 
\indent	It is important to note that Problem $(\ref{pb:classical})$ represents a classical model-based control problem, where the parameters are assumed to be known and specified. We will seek an optimal strategy by reinforcement learning in the absence of knowledge of the parameters. The key idea is to incorporate randomization and exploration. That is, one interacts with the unknown environment through a randomized strategy to gather information, thereby obtaining better strategies. Following Wang et al. \cite{wang2020reinforcement}, we introduce the exploratory version of the wealth process. We first transform the real-valued control $u_t$ into a probability density-valued control denoted by $\pi_t$. At every time, the investor samples from $\pi_t$ to obtain an action $u_t$. Then the exploratory version of the wealth is given by
\begin{equation}\label{eq:ep_wealth}
	dW_t^\pi = \Big[r+(\mu -r)\int_{\mathbb{R}}u\pi_t(u)du\Big]W_t^\pi dt+ \sigma W_t^\pi\sqrt{\int_{\mathbb{R}}u^2\pi_t(u)du}dB_t.
\end{equation}
The detailed derivation of (\ref{eq:ep_wealth}) is provided in Appendix \ref{subsec:derivation}. 
To encourage exploration, we introduce an entropy regularizer into the reward function:
\begin{equation*}
	J(t, w; \pi):= \mathbb E\Big\{U_p(W_T^\pi)+\int_{t}^{T}\lambda(s,W_s^\pi)\Big[\int_{\mathbb{R}}H_\beta(\pi_s(u))du\Big]ds \mid W_t=w \Big\},
\end{equation*}
where $\lambda(t,w)>0$ is a measurable primary temperature function representing the weight on exploration and $H_\beta(\cdot)$ is Tsallis entropy of the distribution given by
\begin{equation*}
	H_\beta(z) =\begin{dcases}
		\frac{1}{\beta -1}(z-z^{\beta}),\quad \beta>1,\\
		-z\log z,\quad \beta=1.
	\end{dcases}
\end{equation*}
When $\beta =1$, the Tsallis entropy reduces to the Shannon entropy. Now, we specify the admissible controls set as follows.
\begin{definition}
	A policy $\pi =\{\pi_s, t \leq s \leq T\}$ is called an admissible distributional control if
	\begin{enumerate}
		\item For each $t \leq s \leq T$, $\pi_s \in \mathcal{P}(\mathbb{R})$, where $\mathcal{P}(\mathbb{R})$ denotes the collection of probability density functions of absolutely continuous probability measures on $\mathbb{R}$. 
         \item For each $(t, w) \in [0, T)\times \mathbb{R}_+$, the SDE (\ref{eq:ep_wealth}) admits a unique strong solution $\{W_s^{\pi}, s \in [t, T]\}$, with $W_t^\pi=w$.    
		\item $\pi$  is a mapping from $[t,T]\times \Omega$ to $\mathcal{P}(\mathbb{R})$, such that $\int_{\Lambda}\pi_s(u) du$   is $\mathcal{F}_s$-{progressively  measurable} for any Borel set $\Lambda$ and $\mathbb E\left\{\int_{t}^{T}[\int_{\mathbb{R}}u^2\pi_s(u)du]ds \mid W_t^\pi=w \right\}<+\infty$. 
		\item $\mathbb E \left\{ |U_p(W_T^\pi)|+ \int_{t}^{T}\lambda(s,W_s^\pi)[\int_{\mathbb{R}}|H_\beta(\pi_s(u))|du]ds  \mid W_t^\pi=w \right\} <\infty$.
	\end{enumerate}
\end{definition}
The set of all admissible distributional controls starting at $(t,w)$ is denoted as $\Pi(t,w)$. Note that integrability in condition 3 requires that the mean and variance of the control $\pi_s$ exist and are finite for almost all $s \in [t, T]$. Our goal is to find an admissible control to maximize the reward function:  
\begin{equation}\label{pb_ex}
	V(t,w) = J(t,w;\hat{\pi}) =\max\limits_{\pi \in \Pi(t, w)}J(t,w;\pi),
\end{equation}
where $V$ is optimal value function.
\subsection{How to Choose Primary Temperature Function?} 
In reinforcement learning problems, the primary temperature function $\lambda(\cdot,\cdot)$ cannot generally be arbitrarily chosen, especially in the utility maximization problem. It can affect the well-posedness and solvability of exploratory control problems. We first discuss how to choose the appropriate candidate primary temperature function through some heuristic analysis. We begin by considering the simplest case, that is, $\lambda$ is only time dependent, and the entropy is the Shannon entropy. 
\begin{proposition}\label{pro_well_time}
    	Let $\lambda(t,w)=\lambda(t)$ be a positive continuous-time-dependent primary temperature function and $\beta=1$. The control problem $(\ref{pb_ex})$ is well-posed for $p \leq 0$ but becomes ill-posed for $0<p<1$. Moreover, when $p=0$,  a closed-form optimal strategy given by\begin{equation}\label{pi_0}
		\hat{\pi}(u|t,w)\sim \mathcal{N}\Big(\frac{\mu-r}{\sigma^2},\frac{\lambda(t)}{\sigma^2}\Big),
	\end{equation}
	and the correspondingly optimal value function is 
	\begin{equation*}
		V(t,w)=\log w +\Big[r+\frac{(\mu-r)^2}{2\sigma^2}\Big](T-t)+\int_{t}^{T}\frac{\lambda(s)}{2}\log\Big(\frac{2\pi \lambda(s)}{\sigma^2}\Big)ds.
	\end{equation*}
\end{proposition}
\begin{proof}
	We first show that the ill-posedess when $0<p<1$. Consider the  strategy
	\begin{equation}\label{eq_strategy}
		\pi_s^n\sim \mathcal{N}\Big(0,n^2\Big), \ s \in [t, T].
	\end{equation}
	It is easy to check that (\ref{eq_strategy}) is an admissible control. Through some calculations, we obtain the value of reward function under the above strategy:
	\begin{equation*}
		\begin{aligned}
			J(t,w;\pi^n)
			&=\frac{w^p}{p} \exp\Big\{pr( T-t)+\Big(\frac{p^2-p}{2}\Big)\sigma^2n^2(T-t)\Big\}+ \frac{1}{2}\log(2\pi e  n^2)\int_{t}^{T}\lambda(s)ds.\end{aligned}
	\end{equation*}
	It follows that when $0<p<1$,
	\begin{equation*}
		\lim\limits_{n \to +\infty}J(t,w;\pi^n)=+\infty,
	\end{equation*}
	which implies the ill-posedness. \\
	\indent It remains to show that it is well-posed when $p<0$. Note that $x^p/p \leq \log x$ for any $x>0$ when $p<0$. Then the correspondingly reward function $J^p \leq  J^0$ for $p<0$ where the subscript $p$ emphasizes the difference in the utility function of $J$. Besides, it is easy to check that there exists an admissible control $\pi$ such that $J^p(t,w;\pi)>-\infty$. On the other hand, the problem when $p=0$ is well-posed, and the optimal control is given in (\ref{pi_0}) (See Jiang et al.\cite{jiang2022reinforcement}). Taking the supremum on both sides of $J^p\leq  J^0$ with respect to admissible controls, we complete the proof.
\end{proof}
	In the proof of Proposition \ref{pro_well_time}, we have some interesting observations. First, when $0<p<1$, the ill-posedness of the control problem arises from the temperature parameter function $\lambda$ depending only on time instead of current wealth. Regardless of the level of wealth, we assign equal weight to exploration. This leads to obtaining rewards through excessive exploration rather than effective exploitation. Second,    when $p<0$, although the problem is well-posed, explicit analytical forms are not available for both the optimal value function and the optimal policy. This is because the corresponding HJB equation (will be derived in detail later)
	\begin{equation*}
		v_t +rwv_w-\frac{1}{2}\frac{(\mu-r)^2v_w^2}{\sigma^2v_{ww}}+\frac{\lambda(t)}{2}\log\Big(-\frac{2\pi \lambda(t)}{\sigma^2w^2v_{ww}}\Big)=0	
	\end{equation*}
	with terminal condition $v(T,x)=x^p/p$ is a fully non-linear equation that cannot be solved by dimensionality reduction due to non-homogeneity. Instead, the existence of PDE may be handled using the method of weak solutions,  which is beyond the scope of this paper. Furthermore, even if a solution exists, the weak nature of the solution makes it difficult to parameterize these functions in simpler forms and perform numerical experiments.\\
	\indent The above findings inspire us to consider the issues of over-exploration and homogeneity when selecting the primary temperature function. Therefore, we infer that the candidate primary temperature function $\lambda(\cdot,\cdot)$ is related to wealth and consider the case that $\lambda(t,w)=\gamma w^p$ where $\gamma>0$ is a constant. 
		\begin{remark}
		$\lambda(t,w)=\gamma w^p$ implies the exploration is wealth-dependent. Besides, the following will show that dimensionality reduction can be applied to the HJB equation so that a (semi-)closed-form solution becomes possible. In addition, when $0<p<1$, it has economic significance, meaning that investors with more wealth will place greater emphasis on exploration. This aligns with real practice, where wealthier individuals tend to engage in higher-risk investments.
	\end{remark}
	\section{Optimal Strategy}\label{sec:optimal_strategy}
	In this section, We will solve for the optimal strategy using dynamic programming in the case of $\lambda(t,w)=\gamma w^p$ and provide several specific examples where (semi-)closed-form solutions exist. For the purpose of comparison, we first present the result of the classical benchmark. 
	\begin{proposition}
		When $\gamma \equiv 0 $, the optimal strategy of control problem $(\ref{pb_ex})$ is 
		\begin{equation}\label{cl_strategy}
			\hat{u}(t,w)=\begin{dcases}
				\frac{\mu -r}{\sigma^2(1-p)}, &\text{if }~p<1, p\neq0,\\
				\frac{\mu-r}{\sigma^2}, &\text{if }~ p=0.
			\end{dcases}
		\end{equation}
		Correspondingly, the optimal value function denoted by $V^{\text{cl}}$  is 
		\begin{equation*}
			V^{\text{cl}}(t,w)=\begin{dcases}
				\exp\Big\{\Big[\frac{(\mu-r)^2 p}{2\sigma^2(1-p)}+rp\Big](T-t)\Big\}\frac{w^p}{p},&\text{if }~p<1, p\neq0,\\
				\Big[\frac{(\mu-r)^2}{2\sigma^2}+r\Big](T-t)+\log w, &\text{if }~p=0.
			\end{dcases}
		\end{equation*}
	\end{proposition}
	Apparently, (\ref{cl_strategy}) also gives the solution to Merton's problem (\ref{pb:classical}). When $\gamma>0$, by the principle of dynamic programming, we derive the exploratory HJB equation satisfied by the optimal value function:
	
		\begin{equation}\label{tsallis_hjb}
		\begin{aligned}0=v_{t}+rw v_{w}+\sup\limits_{\pi}\Big\{\int_{\mathbb{R}}&\Big[(u-r)wv_{w}u+\frac{1}{2}\sigma^{2}w^{2}v_{ww}u^{2}\Big]\pi(u)du\\&+\gamma w^p\int_{\mathbb{R}}H_\beta(\pi(u))du\Big\},\end{aligned}
	\end{equation}
	with terminal condition $v(T,w)=w^p/p$.\\
 	\indent	Let $Y(u;\pi)=\Big((\mu -r)uwv_w+\frac{1}{2}\sigma^2u^2w^2v_{ww}\Big)\pi(u)-\gamma w^pH_\beta(\pi(u))$. Consider the following optimization problem:
 \begin{equation*}
 	\begin{cases}
 		\max\limits_{{\pi}\in \mathcal{P}(\mathbb R)}\int_{\mathbb R} Y(u;\pi)du,\\
 		s.t.\ \int_{\mathbb R}\pi(u)du=1~ \text{and}~ \pi(u)\geq 0.
 	\end{cases}
 \end{equation*}
 We introduce Lagrange multipliers $\varphi(t,w)$ and $\xi(t,w,u)$, and consider the following optimization problem:
 \begin{equation*}
 	\max_{{\pi}\in \mathcal{P}(\mathbb R)} L(u,\pi),
 \end{equation*}
 where $L$ is defined as 
 $$L(u,\pi) = \int_{\mathbb{R}}Y(u,\pi)+\varphi(t,w)\Big(\int_{\mathbb{R}}\pi(u)du-1\Big)+\int_{\mathbb{R}}\xi(t,w,u)\pi(u)du,$$
and
\begin{eqnarray}{}
& \int_{\mathbb R}\pi(u)du=1 \label{-psi},\\
 & \xi(t,w,u)\pi(u)=0~ \text{and}~ \pi(u)\geq 0. \label{-xi}
\end{eqnarray}
According to the first variation principle, we can derive the corresponding Euler-Lagrange equation. Simplifying it, we obtain
 \begin{equation}\label{ex_xi}
 		\hat{\pi}(u|t,w)=
 		\begin{dcases}
 			\Big(\frac{\beta-1}{\beta \gamma w^p }\Big)^{\frac{1}{\beta-1}}\Big(\varphi(t,w)+ \xi(t,w,u)+\mathcal{A}v+\frac{\gamma w^p}{\beta-1}\Big)^{\frac{1}{\beta-1}},\quad \beta>1,\\
 			\exp\Big\{\varphi(t,w)+ \xi(t,w,u)+\mathcal{A}v-\gamma w^p\Big\},\quad \beta=1,
 		\end{dcases}
 \end{equation}
 where $\mathcal{A}v=(\mu -r)uwv_w+\frac{1}{2}\sigma^2u^2w^2v_{ww}$. By (\ref{-xi}), we have
 \begin{equation}\label{xi}
 	\xi(t,w,u)= 
 	\begin{dcases}
 		\Big(\varphi(t,w)+\mathcal{A}v+\frac{\gamma w^p}{\beta-1}\Big)_-, \quad \beta>1,\\
 		0,\quad \beta=1,	
 	\end{dcases}
 \end{equation}
 where $(x)_-=-\min\{x,0\}$. Substituting (\ref{xi}) into (\ref{ex_xi}) yields
 
 \begin{equation}\label{pi_tsalli}
 	\hat{\pi}(u|t,w)=
 	\begin{dcases}
 		\Big(\frac{\beta-1}{\beta \gamma w^p }\Big)^{\frac{1}{\beta-1}}\Big(\varphi(t,w)+ \mathcal{A}v+\frac{\gamma w^p}{\beta-1}\Big)_+^{\frac{1}{\beta-1}},\quad \beta>1,\\
 		\exp\Big\{\varphi(t,w)+\mathcal{A}v-\gamma w^p\Big\},\quad \beta=1.
 	\end{dcases}
 \end{equation}
 where $(x)_+=\max\{x,0\}$. $\varphi$ is chosen to satisfy normalization condition (\ref{-psi}), that is, 
 \begin{equation}\label{varphi}
 	\int_{\mathbb{R}}\hat{\pi}(u|t, w)du=1.
 \end{equation}
For a general $\beta$, it is difficult to obtain an (semi-)explicit expression for $\varphi$, let alone for optimal strategy $\hat{\pi}$ and optimal value function $V$. Next, we will study two specific examples.
\subsection{The Case where $\beta=1$}	
When $\beta=1$, Tsallis entropy degenerates to the well-known Shannon entropy. According to (\ref{varphi}), we can compute
\begin{equation}\label{varphi_1}
	\varphi(t,w)=\log\frac{\exp(\gamma w^p)}{\int_{\mathbb{R}}\exp(\mathcal{A}v)du}.
\end{equation}
Substituting (\ref{varphi_1}) into (\ref{pi_tsalli}), we get
	 \begin{equation}\label{eq:optimal_p}
	 	\begin{split}
	 		\hat{\pi}(u|t,w)\sim \mathcal{N}\Big(-\frac{(\mu-r)v_w}{\sigma^2wv_{ww}},-\frac{\gamma w^{p-2}}{\sigma^2v_{ww}}\Big),
	 	\end{split}
	 \end{equation}
	where we assume $v_{xx}<0$ which will be verified later. The corresponding HJB equation is reduced to 
	\begin{equation}\label{hjb_beta1}
		\begin{dcases}
			v_t+rwv_w-\frac{1}{2}\frac{(\mu-r)^2v_w^2}{\sigma^2v_{ww}}+\frac{\gamma w^p}{2}\log\Big(-\frac{2\pi \gamma w^{p-2}}{\sigma^2v_{ww}}\Big)=0,\\
			v(T,w)=\frac{w^p}{p}.
		\end{dcases}
	\end{equation}
	Noting that this equation exhibits homogeneity, we consider dimensional reduction. We make an ansatz that $v(t,w)=f(t)w^p/p$. Plugging it into (\ref{hjb_beta1}) yields 
	\begin{equation}\label{eq:ODE_f}
		\begin{dcases}
			f'(t)=\Big[-\frac{p (\mu-r)^2}{2(1-p)\sigma^2}-rp\Big]f(t)+\frac{p \gamma}{2} \log(f(t))+\frac{p \gamma}{2}\log\Big(\frac{\sigma^2(1-p)}{2 \pi \gamma}\Big),\\ 
			f(T)=1.
		\end{dcases}
	\end{equation}
	Therefore, we only need to discuss the existence and uniqueness of the ODE (\ref{eq:ODE_f}). Let $y(t)=f(T-t)$, $h(y) = ay+b\log(y)+c$ where
	$$a=rp+\frac{p (\mu-r)^2}{2(1-p)\sigma^2}, ~~b=-\frac{p \gamma}{2}, ~~c=-\frac{p \gamma}{2}\log\Big(\frac{\sigma^2(1-p)}{2 \pi \gamma}\Big).$$ Then (\ref{eq:ODE_f}) is equivalent to 
\begin{equation}\label{eq:ODE2}
	\begin{dcases}
		y'=h(y),\\ 
		y(0)=1.
	\end{dcases}
\end{equation}
	Note that since $(r,\mu,\sigma,\gamma)\in \mathbb{R}\times\mathbb{R}\times \mathbb{R}_+\times \mathbb{R}_+$ and $p<1,p\neq 0$, then $(a,b,c)$ takes value in $ \mathbb{R}\times \mathbb{R}/{\{0\}}\times \mathbb{R}. $
	We first study the properties of $h$. 
	\begin{lemma}\label{lemma}
		 We have the following results for function $h$:
		\begin{enumerate}
			\item Consider the case where $p<0$ which implies $b>0$.\\
			$(i)$ If $a\geq 0$, then $h(y)$ is a strictly increasing function in $(0,+\infty)$, with values ranging from $-\infty$ to $\infty$.\\
			$(ii)$ If $a<0$ and $b[\log b-\log(-a)-1]+c>0$, then $h(y)$ initially increases and then decreases  in $(0,+\infty)$, with two distinct zeros.\\
		$	(iii)$ If $a<0$ and $b[\log b-\log(-a)-1]+c\leq 0$, then $h(y)$ initially increases and then decreases  in $(0,+\infty)$ with $h(y)\leq 0$. Moreover, if $b[\log b-\log(-a)-1]+c= 0$, there exists a unique zero.
			\item Consider the case where $0<p<1$, which implies $b<0$.\\
            $	(i)$ If $a\leq 0$, then $h(y)$ is a strictly decreasing function in $(0,+\infty)$, with values ranging from $-\infty$ to $\infty$.\\
            $(ii)$ For $a>0$, if $b[\log(- b)-\log(a)-1]+c<0$, then $h(y)$ initially decreases and then increases in $(0,+\infty)$, with two distinct zeros\\
			$ (iii)$For $a>0$, if $b[\log(- b)-\log(a)-1]+c\geq 0$, then $h(y)$ exhibits the same monotonicity as $(ii)$ but $h(y)\geq 0$. Moreover, if $b[\log(- b)-\log(a)-1]+c = 0$, there exists a unique zero.
		\end{enumerate}
	\end{lemma}
	\begin{proof}
		The proof is easy. We omit it.
	\end{proof}
	 Combining the Lemma \ref{lemma} and the theory of ODE, we have:
	\begin{proposition}\label{them}~
		\begin{enumerate}
			\item When $0<p<1$, the ODE $(\ref{eq:ODE2})$ admits a unique \textbf{positive} solution defined in $[0,+\infty)$. Specifically, the solution falls into one of the  four cases:\\
			$ (i)$ $y(\cdot)$ is increasing  to $+\infty$; \\
			 $(ii)$ $y(\cdot)$ is increasing and converges to a constant greater than 0;\\
			 $(iii)$ $y(\cdot)$ is decreasing and converges to a constant greater than 0;\\
			 $ (iv)$ $y(\cdot)$ is identically equal to a constant.  
			\item When $p<0$, the ODE $(\ref{eq:ODE2})$ admits a unique solution defined in $[0,\delta)$. $\delta$ is finite or positive infinity. When $0<\delta<+\infty$, $y(\cdot)$ decreases to $0$ at $\delta$. When $\delta=+\infty$, h either increases or decreases and in both cases converges to a constant greater than 0, or $h$ is identically constant. 
		\end{enumerate}
	\end{proposition}
	\begin{proof}
		The proof is provided in Appendix \ref{pf_them}.
	\end{proof}
    \begin{figure}[htbp]
		\centering
		\includegraphics[width=1\linewidth]{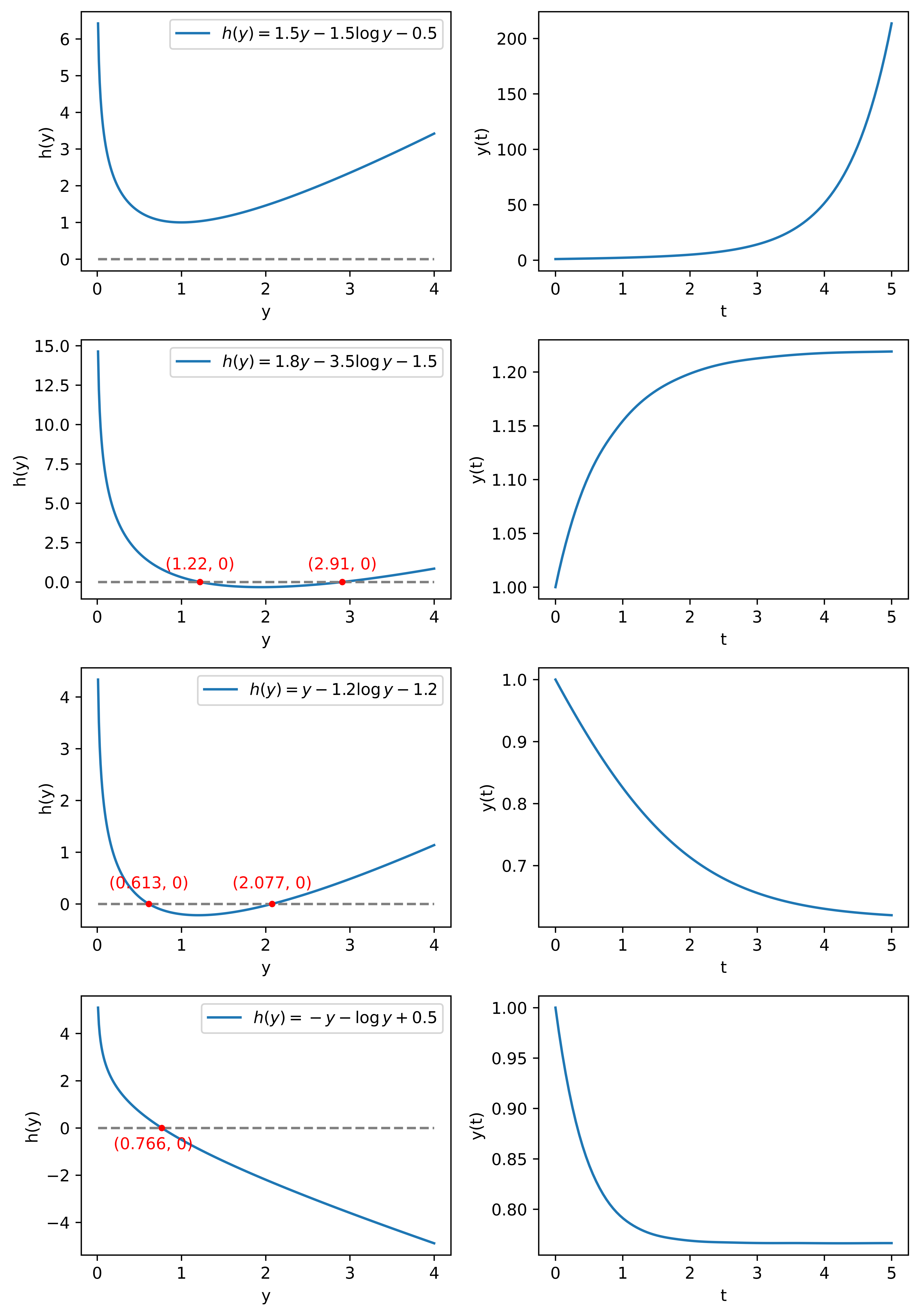}
		\caption{Numerical Solutions of the ODE under Different Parameter Selections when $b<0$}\label{fig:ode1}
	\end{figure} 
	\begin{figure}[htbp]
		\centering
		\includegraphics[width=1\linewidth]{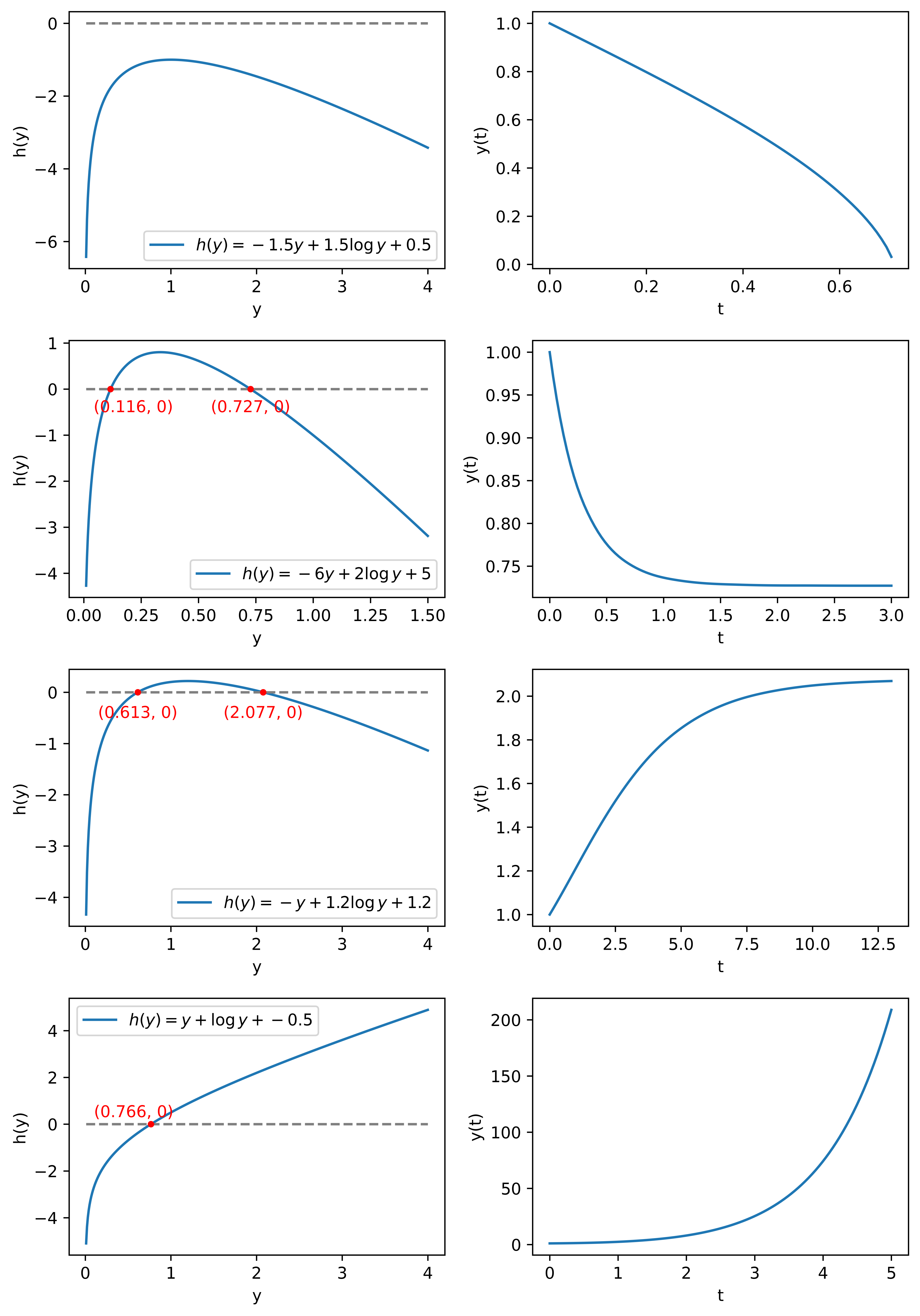}
		\caption{Numerical Solutions of the ODE under Different Parameter Selections when $b>0$}\label{fig:ode2}
	\end{figure} 
	For intuitive understanding of the reader, we employ the finite difference method (cf. Boyce et al. \cite{boyce2017elementary}) to numerically solve the ODE (\ref{eq:ODE2}) and present several representative numerical examples in Figure \ref{fig:ode1} and Figure \ref{fig:ode2}. The left column in the two Figures shows the graph of $h(\cdot)$, while the right column presents the solution of the corresponding ODE. 
    Next, we discuss the well-posedness based on Proposition \ref{them}.
\begin{proposition}\label{coro_wellpose}
We consider the exploratory problem $(\ref{pb_ex})$ with $\lambda(t,w)=\gamma w^p, \beta=1$, and $(r,\mu,\sigma,\gamma)$ taking values in $\mathbb{R}\times\mathbb{R}\times \mathbb{R}_+\times \mathbb{R}_+$. 
\begin{itemize}
\item[Case 1:] when $0<p<1$, we have
$v(t,w)<+\infty$, for $(t, w) \in [0, T] \times\mathbb{R}_+ $.
\item[Case 2:] when $p<0$ and the ODE $(\ref{eq:ODE2})$ admits a solution in $[0,\delta)$ where $\delta$ is defined in Proposition \ref{them}.
\begin{itemize}
\item[Case 2.1:] If $\delta>T$, we have $v(t,w)<+\infty$, for $(t, w) \in [0, T] \times\mathbb{R}_+ $.
\item[Case 2.2:] If $\delta<T$, we have $v(t,w)<+\infty$, for $(t, w) \in [T-\delta, T] \times\mathbb{R}_+ $.
\item[Case 2.3:] If $\delta<T$, we have $v(t,w)=+\infty$, for $(t, w) \in [0, T-\delta) \times\mathbb{R}_+ $.
\end{itemize}
\end{itemize}
\end{proposition}
\begin{proof}
	The proof is given in Appendix \ref{proof_coro_wellpose}.
\end{proof}

Based on the discussions above, we summarize the verification theorem.
\begin{theorem}\label{ver_them}$(\textbf{Verification Theorem})$ We consider the exploratory problem $(\ref{pb_ex})$ with $\lambda(t,w)=\gamma w^p$ and $\beta=1$, where $(r,\mu,\sigma,\gamma)$ taking values in $\mathbb{R}\times\mathbb{R}\times \mathbb{R}_+\times \mathbb{R}_+$ such that either Case 1, 2.1 or 2.2 holds.
	Then the optimal policy of the exploratory problem $(\ref{pb_ex})$ is given by \begin{equation*}
			\begin{split}
				\hat{\pi}(u|t,w)\sim \mathcal{N}\Big(\frac{(\mu-r)}{\sigma^2(1-p)},\frac{\gamma }{(1-p)\sigma^2f(t)}\Big).
			\end{split}
		\end{equation*}
        where $f>0$ defined in corresponding time interval $(\tau, T]$, where $\tau=T-\delta\in [0, T)$, is a classical solution of ODE $(\ref{eq:ODE_f})$. 
	Furthermore, the optimal value function is $V(t,w)=f(t)x^p/p$.
\end{theorem}
	\begin{proof}
		The proof is given in Appendix \ref{proof_ver_them}.
	\end{proof}
	From Theorem \ref{ver_them}, we can see that the optimal control is presented in the form of feedback control, which can generate open-loop control. Beside, the mean of the optimal strategy is the same as that of the classical Merton strategy. This suggests that to obtain an exploratory strategy whose mean aligns with that of the classical counterpart, we must carefully design our model, particularly the primary temperature function.	
	Define the relative exploratory cost function as
	$$C^{\gamma}(t,w):=\Bigg|\frac{\mathbb{E}[U_p(W_T^{\hat{u}})\mid W_t=w]-\mathbb{E}[U_p(W_T^{\hat{\pi}})\mid W_t=w]}{\mathbb{E}[U_p(W_T^{\hat{u}})\mid W_t=w]}\Bigg|,$$
where $\hat{u}$ is the solution of Problem (\ref{pb:classical}).	$C^{\gamma}$ represents the difference between the expected utility of terminal wealth under the optimal strategy in the classical and exploratory versions. According to Theorem \ref{ver_them}, we have
	\begin{corollary}\label{coro}
		The relative exploration cost of the utility problem due to the lack of information about market is
		$$C^{\gamma}(t,w) =\Big|1-\exp\Big\{-\frac{p\gamma}{2}\int_{t}^{T}\frac{ds}{f(s)} \Big\}\Big|.$$
	\end{corollary}
	 \begin{proof}
	 	The proof is trivial. We omit it.
	 \end{proof}
	 Next, we investigate the convergence of the solution and strategy as $\gamma \to 0^+$ given the well-posedness. We focus exclusively on the case where $0<p<1$, as this is the scenario in which well-posedness is guaranteed and a semi-closed strategy exists, regardless of the choice of market parameters. In the following discussion, we will use the subscript $\gamma$ to indicate the dependence on $\gamma$. We first present the following lemma, which is beneficial for ensuring convergence.
	 
	 \begin{lemma}\label{pro_conv}
     Consider $0<p<1$.
	 	\begin{equation}
	  \lim\limits_{\gamma \to 0+}\max_{t\in [0,T]}|y^\gamma(t)-y^0(t)|=0,
	 	\end{equation}
	 	where $y^0(\cdot)$ represents the solution of ODE $(\ref{eq:ODE2})$ when $\gamma=0$,  given explicitly as 
	 	$$y^0(t)=\exp\Big\{\Big[\frac{(\mu-r)^2 p}{2\sigma^2(1-p)}+rp\Big](T-t)\Big\}.$$
	 \end{lemma}
	 \begin{proof}
	 	The proof is given in Appendix \ref{proof_pro_conv}.
	 \end{proof}
	 Thanks to Lemma \ref{pro_conv}, we can derive the desired convergence result.
	 \begin{proposition}
      Consider $0<p<1$.
	  As $\gamma \to 0^+$, $V^\gamma\to V^{\text{cl}}$ locally uniformly , $C^\gamma \to 0$ uniformly, $\hat{\pi}^\gamma_t(\cdot|t,w) \to \delta_{\hat{u}(t,w)}(\cdot)$ weakly for any $t\in[0,T]$.
	 	
	 \end{proposition}
	 \begin{proof}
	 	The characteristic function of $\hat{\pi}^\gamma_t$ is given by
	 	$$\psi^\gamma(\xi)=\int_{\mathbb{R}}\exp(i\xi u)\hat{\pi}_t^\gamma(u)du=\exp\Big\{\frac{i\xi(\mu-r)}{\sigma^2(1-p)}-\frac{\gamma \xi^2}{2(1-p)\sigma^2f^\gamma(t)}\Big\}.$$
	 	By Lemma \ref{pro_conv}, for any $t\in [0,T]$, as $\gamma\to 0^+$
	 	$$\psi^\gamma(\xi)\to\exp\Big\{\frac{i\xi(\mu-r)}{\sigma^2(1-p)}\Big\} , $$
	 	which is the characteristic function of $\delta_{\hat{u}(t,w)}$. Thus, $\hat{\pi}^\gamma_t(\cdot|t,w) \to \delta_{\hat{u}(t,w)}(\cdot)$ weakly for any $t\in[0,T]$. The remaining proofs are straightforward.
	 \end{proof}
	\subsection{The Case where $\beta=3$}
	We consider the exploratory HJB equation with respect to Tsalli entropy for $\beta =3$.
We make an ansatz $v(t,w)=f(t)w^p/p$ and assume $f> 0$ (which will be verified later). Plugging the expression of $v$ into the optimal strategy (\ref{pi_tsalli}), we have
	\begin{equation}\label{pi_v}
		\hat{\pi}(u|t,w)=\Big(\frac{2}{3\gamma w^p }\Big)^{\frac{1}{2}}\Big(\varphi(t,w)+(\mu-r)w^pf(t)u+ \frac{1}{2} \sigma^2(p-1)w^pf(t)u^2+\frac{\gamma w^p}{2}\Big)_+^\frac{1}{2}.
	\end{equation}
	For simple notations, we define
	\begin{equation*}
		\begin{dcases}
			A=\frac{1}{2} \sigma^2(1-p)w^pf >0,\\
			F=\Big(\frac{2}{3\gamma w^p }\Big)^{\frac{1}{2}}\sqrt{A}>0,\\
			R = \frac{\varphi}{A}+\frac{\gamma w^p}{2A}+\frac{(\mu-r)^2w^pf^2}{2\sigma^2(1-p)A}>0,\\
			\theta=\frac{\mu-r}{\sigma^2(1-p)}.
		\end{dcases}
	\end{equation*}
	After tedious calculations, (\ref{pi_v}) can be simplified to
	\begin{equation*}
		\hat{\pi}(u|t,w)=F\sqrt{\Big(R^2-(u-\theta)^2\Big)},\quad \theta-R \leq u \leq \theta +R.
	\end{equation*}
	By normalized condition $\int_{\mathbb{R}}\hat{\pi}(u|t, w)du=1$, we can get
	$$\varphi(t,w)=\frac{\sqrt{6\gamma w^pA}}{\pi }- \frac{\gamma w^p}{2}-\frac{(\mu-r)^2w^pf^2}{2\sigma^2(1-p)}. $$
	Similarly, we can compute
	\begin{equation}\label{pi_moment}
		\begin{dcases}
		\int_\mathbb{R}u \hat{\pi}(u)du= \theta,\\
		\int_\mathbb{R}u^2 \hat{\pi}(u)du=\frac{FR^4\pi}{8}+\theta^2,	\\
		\int_\mathbb{R}\hat{\pi}^3(u)du=\frac{3F^3R^4 \pi}{8}.
		\end{dcases}
	\end{equation}

	Substituting (\ref{pi_moment}) into the exploratory HJB equation (\ref{tsallis_hjb}) and performing tedious calculations, we can obtain the ODE
	\begin{equation}\label{ode_tsalli}
	 \frac{f'(t)}{p}+\Big[r+\frac{1}{2}\frac{(\mu-r)^2}{\sigma^2(1-p)}\Big]f(t)-\frac{\sqrt{3(1-p) \sigma^2\gamma}}{2\pi}\sqrt{f(t)}+\frac{\gamma}{2}=0,
	\end{equation}
	with terminal condition $f(T)=1$. Let $y(T-t)=f(t)$, then ODE (\ref{ode_tsalli}) is equivalent to 
	\begin{equation}\label{reverse_ode_tsalli}
		y'= ay+b\sqrt{y}+c,\quad y(0)=1,
	\end{equation}
	where 
	\begin{equation*}
		\begin{dcases}
			a=p\Big[r+\frac{1}{2}\frac{(\mu-r)^2}{\sigma^2(1-p)}\Big],\\
			b=-\frac{p\sqrt{3(1-p) \sigma^2\gamma}}{2\pi},\\
			c=\frac{\gamma p}{2}.
		\end{dcases}
	\end{equation*}
	Considering (\ref{reverse_ode_tsalli}), we have
	\begin{proposition}\label{prop_tsalli_beta3}~
	\begin{enumerate}
		\item When $0<p<1$, the ODE $(\ref{reverse_ode_tsalli})$ admits a unique \textbf{positive} solution defined on $[0,+\infty)$. 
		\item When $p<0$, the ODE $(\ref{reverse_ode_tsalli})$ admits a unique solution defined in $[0,\delta)$. $\delta$ is finite or positive infinity. When $0<\delta<+\infty$, $y(\cdot)$ decreases to $0$ at $\delta$. 
	\end{enumerate}
\end{proposition}
\begin{proof}
	The proof is similar to the proof of Proposition \ref{them}. We omit it.
\end{proof}
The result of Proposition \ref{prop_tsalli_beta3} is very similar to that of Proposition \ref{them}. That is, when $0<p<1$, the control problem is well-posed and always admits an optimal strategy. However, when $p<0$, the problem can be ill-posed. The well-posedness of the control problem with $\beta=3$ is similar to that in Proposition \ref{coro_wellpose}, while the optimal strategy is given by the following.
\begin{theorem}We consider the exploratory problem $(\ref{pb_ex})$ with $\lambda(t,w)=\gamma w^p, \beta=3$. Suppose that ODE $(\ref{ode_tsalli})$ admits a classical solution $f>0$ defined in $(\tau, T], \tau\in [0, T)$. Then the optimal distributional strategy is given by, for $t \in (\tau,  T]$,
	\begin{equation}\label{pi_beta3}
		\hat{\pi}(u|t,w)=\sqrt{\frac{\sigma^2(1-p)f(t)}{3\gamma}}\sqrt{\frac{2}{\pi}\sqrt{\frac{3\gamma}{\sigma^2(1-p)f(t)}}-\Bigg(u-\frac{\mu-r}{\sigma^2(1-p)}\Bigg)^2},
	\end{equation}
and the optimal value function is $V(t,w)=f(t)w^p/p$.
\end{theorem}
Under the Tsallis entropy, the optimal distribution strategy is no longer a normal distribution. Instead, it is the so-called Wigner semicircle distribution defined on a compact support. Additionally, it can be regarded as a scaled Beta distribution. Specifically, assume that $Z$ follows a Beta distribution with shape parameters all $3/2$, i.e., $Z\sim \text{Beta}(3/2,3/2)$, then the random variable
$$\sqrt{\frac{2}{\pi}}\Big[\frac{3\gamma}{\sigma^2(1-p)f(t)}\Big]^{\frac{1}{4}}Z-\sqrt{\frac{2}{\pi}}\Big[\frac{3\gamma}{\sigma^2(1-p)f(t)}\Big]^{\frac{1}{4}}+\frac{\mu-r}{\sigma^2(1-p)}$$
has density function (\ref{pi_beta3}).

 Further calculation gives us its mean and variance of the optimal strategy as $\frac{\mu-r}{\sigma^2(1-p)}$ and $\frac{1}{2\pi}\sqrt{\frac{3\gamma}{\sigma^2(1-p)f(t)}}$, respectively.
An interesting observation is that the mean of the optimal distributional strategy remains the same as the classical optimal strategy (\ref{cl_strategy}).
\section{RL Algorithm Design}\label{sec:RL}
	In this section we mainly focus on designing the reinforcement learning algorithm that corresponds to the Shannon entropy ($\beta=1$) and $0<p<1$. The other case where $\beta=3$ is similar and we leave it for interested readers. Unlike Dai et al. \cite{dai2023learning} and Jiang et al. \cite{jiang2022reinforcement} who adopt the MSTDE approach (mean square temporal difference error), we use the new actor-critic methods for continuous time developed by Jia and Zhou \cite{jia2022policy,jia2022policy_a}.
	\subsection{Parametrization of value function and policy}
	Thanks to Theorem \ref{them}, we introduce the following parameterization for optimal value function:
	\begin{equation}
	V^\theta(t,w)=f^\theta(t)\frac{w^p}{p},
	\end{equation}
	where $\theta$ represents the parameters of $V$. Since $f>0$ and $f(T)=1$, we choose the following form of function:
	$$f^\theta(t)=\exp\Big(\sum_{i=1}^{n}\theta_i(T-t)^i\Big).$$
	$\theta=(\theta_1,\theta_2,\cdots,\theta_n)\in R^n$. The benefits of this choice are shown in subsequent subsections. 
	\begin{remark}
		An alternative way is to engage neural networks. For example,
		\begin{equation*}
			f^\theta(t)=\exp\Big(NN^\theta(t)\Big),
		\end{equation*}
		where $NN^\theta$ is a neural network. 
	\end{remark}Similarly, we parameterize our policy as\begin{equation}\label{para_pi}
		\pi^\varphi(u|,t,w)\sim \mathcal{N}\Big(\frac{\varphi_1}{1-p}, \frac{\gamma e^{\varphi_2}}{(1-p)f^\theta(t)}\Big),
	\end{equation}
	where $\varphi=(\varphi_1,\varphi_2)\in {\mathbb R}^2$.
	
	\subsection{Data generation}
 Consider an equal partition of the time interval $[0,T]$ given by $0=t_0<t_1<\cdots<t_{N}=T$ with $t_i=i\Delta t$. Assume that we have the stock price data $\{S_{t_i}\}_{i=0}^{N}$ that is either collected from the real market or generated by simulation according to (\ref{eq:GBM}). At each time $t_i$ with $ i=0,\cdots,N-1$, we sample the probability density function $\pi_{t_i}^\varphi$ to obtain an action $u_{t_i}$, the wealth $W_{t_{i+1}}$ can be computed by\begin{equation}\label{compute_wealth}
		W_{t_{i+1}}-W_{t_i}=rW_{t_i}\Delta t +u_{t_i}W_{t_i}\frac{S_{t_{i+1}}-S_{t_i}}{S_{t_i}}.
	\end{equation}
	Continuing this procedure, we obtain a full trajectory of states denoted by $(t_i,u_{t_i},W_{t_i})_{i=0}^N$ under the policy $\pi$, which is called an episode in the RL literature.
	
	\subsection{Policy Evaluation}
	Given a policy $\pi^\varphi$, we collect a sequence of data on the dynamics of wealth. Now we would utilize this data set to estimate the value function $V^\theta$. That is, choose parameters $\theta$ such that $V^\theta$ best approximate $V$. We adopt the martingale loss method proposed by Jia and Zhou \cite{jia2022policy_a} to estimate $\theta$, which is an offline algorithm that requires a full trajectory. For simple notation, we let $\mathcal{H}(\pi)=\int_{\mathbb{R}}H(\pi(u))du$. Denote
	$$M_t:=V^\theta(t,W_t^\pi)+\gamma\int_{0}^{t}(W^\pi_s)^p\mathcal{H}(\pi_s)ds,$$
	which is a martingale. Let $M^\theta$ be the parameterized version of $M$. The martingale loss function is 
	\begin{equation*}\label{eq:ML}
		\begin{split}
				\text{ML}(\theta)&=\frac{1}{2}E\Big[\int_{0}^{T}|M^\theta_T-M_t^\theta|^2dt\Big]\\
				&=\frac{1}{2}E\int_{0}^{T}\Big[U_p(W_T^{\pi^\varphi})-V^\theta(t,W_t^{\pi^\varphi})+\gamma\int_{t}^{T}(W_s)^p\mathcal{H}(\pi_s^\varphi)ds\Big]^2dt\\
				&\approx:\frac{1}{2}E\Big[\sum\limits_{i=0}^{N-1} D_i^2\Delta t\Big],
		\end{split}
	\end{equation*}
where 
$$D_i=U_p(W_T)-V^\theta(t_i,W_{t_i})+\gamma\sum\limits_{j=i}^{N-1}(W_{t_j})^p\mathcal{H}(\pi_{t_j}^\varphi)\Delta t.$$
We want to minimize the martingale loss function (\ref{eq:ML}) with respect to $\theta$. The gradient of martingale loss function is 
\begin{equation}\label{grad_theta}
	\frac{\partial}{\partial \theta}\text{ML}(\theta):\approx -E\Big[\sum\limits_{i=0}^{N-1} \Big(D_i\times\frac{\partial}{\partial \theta}V^\theta_i\Big)\Delta t\Big],
\end{equation}
where 
\begin{equation}\label{exp_PE}
	\begin{split}
		\frac{\partial}{\partial \theta}V^\theta_i&=\frac{\partial}{\partial \theta}V^\theta(t_i,W_{t_i})\\
		&=V^\theta(t_i,W_{t_i})\cdot [T-t_i,(T-t_i)^2,\cdots,(T-t_i)^n].
	\end{split}
\end{equation}
Thus, we can update $\theta$ using stochastic gradient descent (SGD).
\begin{remark}
	It is not necessary to wait until the policy evaluation is complete before performing the policy improvement. Instead, we can perform one round of value updates during policy evaluation and then proceed directly to policy improvement based on updated values.
\end{remark}

\subsection{Policy Update}
	Following Jia and Zhou \cite{jia2022policy}, we first compute the policy gradient: \begin{equation}\label{grad_varphi}
		\begin{aligned}
			\frac{\partial}{\partial \varphi}V^{\pi^\varphi}(0,w_0)&=E\Big\{
		\int\limits_{0}^{T}\frac{\partial \log\pi^{\varphi}_s}{\partial\varphi}\left[dV_s^{\pi^{\varphi}}+\gamma(W_s)^p\mathcal{H}(\pi_s^\varphi)ds\right]+\gamma(W_s)^p\frac{\partial}{\partial \varphi}\mathcal{H}(\pi_s^\varphi)ds\Big\}\\
		&\approx : E\sum\limits_{i=0}^{N-1}G_i,
		\end{aligned}
	\end{equation}
	where
	\begin{equation*}
		G_i:= \frac{\partial\log\pi^\varphi_{t_i}}{\partial\varphi}\Big[V^\theta_{t_{i+1}}-V^\theta_{t_i}+\gamma(W_{t_i})^p\mathcal{H}(\pi^\varphi_{t_i})\Delta t\Big]+\gamma(W_{t_i})^p\frac{\partial}{\partial\varphi}\mathcal{H}(\pi^\varphi_{t_i})\Delta t.
	\end{equation*}
	According to (\ref{para_pi}), we can compute that
	\begin{equation}\label{exp_PG}
		\begin{dcases}
			\frac{\partial\log \pi^\varphi_{t_i}}{\partial\varphi_1}&=\frac{e^{-\varphi_2}f^\theta(t_i)}{\gamma}\Big(u_{t_i}-\frac{\varphi_1}{1-p}\Big),\\
				\frac{\partial\log \pi^\varphi_{t_i}}{\partial\varphi_2}&=\frac{(1-p)e^{-\varphi_2}f^\theta(t_i)}{2\gamma}\Big(u_{t_i}-\frac{\varphi_1}{1-p}\Big)^2-\frac{1}{2},\\
				\frac{\partial \mathcal{H}( \pi^\varphi_{t_i})}{\partial\varphi_1}&=0,\\
				\frac{\partial \mathcal{H}( \pi^\varphi_{t_i})}{\partial\varphi_2}&=\frac{1}{2}.
		\end{dcases}
	\end{equation}
	Equations (\ref{exp_PE}) and (\ref{exp_PG}) demonstrate the advantages of setting the parameterized function as an exponential function: it simplifies calculations and saves computational resources. After obtaining the policy gradient, we can update $\varphi$ using gradient ascent. Finally, we summarize all the results in Algorithm \ref{alg}.
		\begin{algorithm}[H]
		\caption{Off-line Reinforcement Learning Algorithm}
		\label{alg}
		\begin{spacing}{1.3}
			\begin{algorithmic}[H]
				\REQUIRE Initial wealth $w_0$, stock price data, risk-free rate $r$, investment horizon $T$, time step $\Delta=T/N$, exploration coefficient $\gamma$, number of iterations $M$, number of mini-batches $\tilde{N}$, learning rate $l_\theta,l_\varphi$.
				\FOR{$k=1$ to $M$ }
				\FOR{$i=0$ to $N-1$ }
				\STATE Sample $\epsilon_i\sim\mathcal{N}(0,I_{\tilde{N}})$.
				
				\STATE Compute the action $u_{t_i}=\frac{\varphi_1}{1-p}+\sqrt{\frac{\gamma \exp(\varphi_2)}{(1-p)f^\theta(t_i)}} \epsilon_i$.
				\STATE Compute the wealth $W_{t_{i+1}}$ using (\ref{compute_wealth}).
			
				\ENDFOR
				\STATE Obtain $\tilde{N}$ full sample trajectories $[(t_i,u_{t_i},W_{t_i}),i=0,\cdots N]$.
				\STATE \textbf{Policy Evaluation}
				\STATE Compute mini-batch gradient $\Delta \theta$ by (\ref{grad_theta}).
		 	    \STATE Update $\theta \leftarrow \theta-l_\theta\Delta\theta$.\\
				\STATE \textbf{Policy Update}
				\STATE Compute mini-batch gradient $\Delta \varphi$ by (\ref{grad_varphi}).
				\STATE Update $\varphi \leftarrow \theta+l_\varphi\Delta\varphi$.\\ 
				
				\ENDFOR
			\end{algorithmic}
		\end{spacing}

	\end{algorithm}
	
	\section{Simulation Result}\label{sec:simulation}
We conduct the numerical experiments according to Algorithm \ref{alg}. We set the default market parameters as shown in Table \ref{tab:parameters} and generate data during the training process using these parameters.
	\begin{table}[htbp]
	\centering
	\caption{Default parameters}
	\label{tab:parameters}
	\begin{tabular}{|c|c|c|c|c|c|c|c|c|}
		\hline
		parameter & $w_0$ & $r$ & $\mu $& $\sigma$ & $T$ &$\Delta t$&$p$&$\gamma$\\ \hline
		value & 1 & 0 & 0.2 & 0.5 & 1 & 1/250  &1/3&0.3 \\ \hline
	\end{tabular}
\end{table}
  The parameterized function $f^\theta$ is chosen to be
$$f^\theta(t)=\exp\Big(\theta_1(T-t)+\theta_2(T-t)^2\Big).$$ In each experiment, we run $M=5000$ episodes with mini-batches of 32 samples. The training parameters $(\theta,\varphi)$ are initialized to zero and the corresponding updating learning rates are set as $l_\theta=0.001$ and $l_\varphi=0.01$. \\
\indent We first focus on the convergence of the algorithm, specifically whether the key parameters
$\varphi_1$ and $\varphi_2$ converge to the corresponding true value $\frac{\mu-r}{\sigma^2} $ and $-\log \sigma^2$. The dynamics of $\varphi$ during the training process of one experiment is shown in Figure \ref{fig:varphi}, where the dashed line represents the true values.
\begin{figure}[htbp]
	\centering
	\includegraphics[width=1\linewidth]{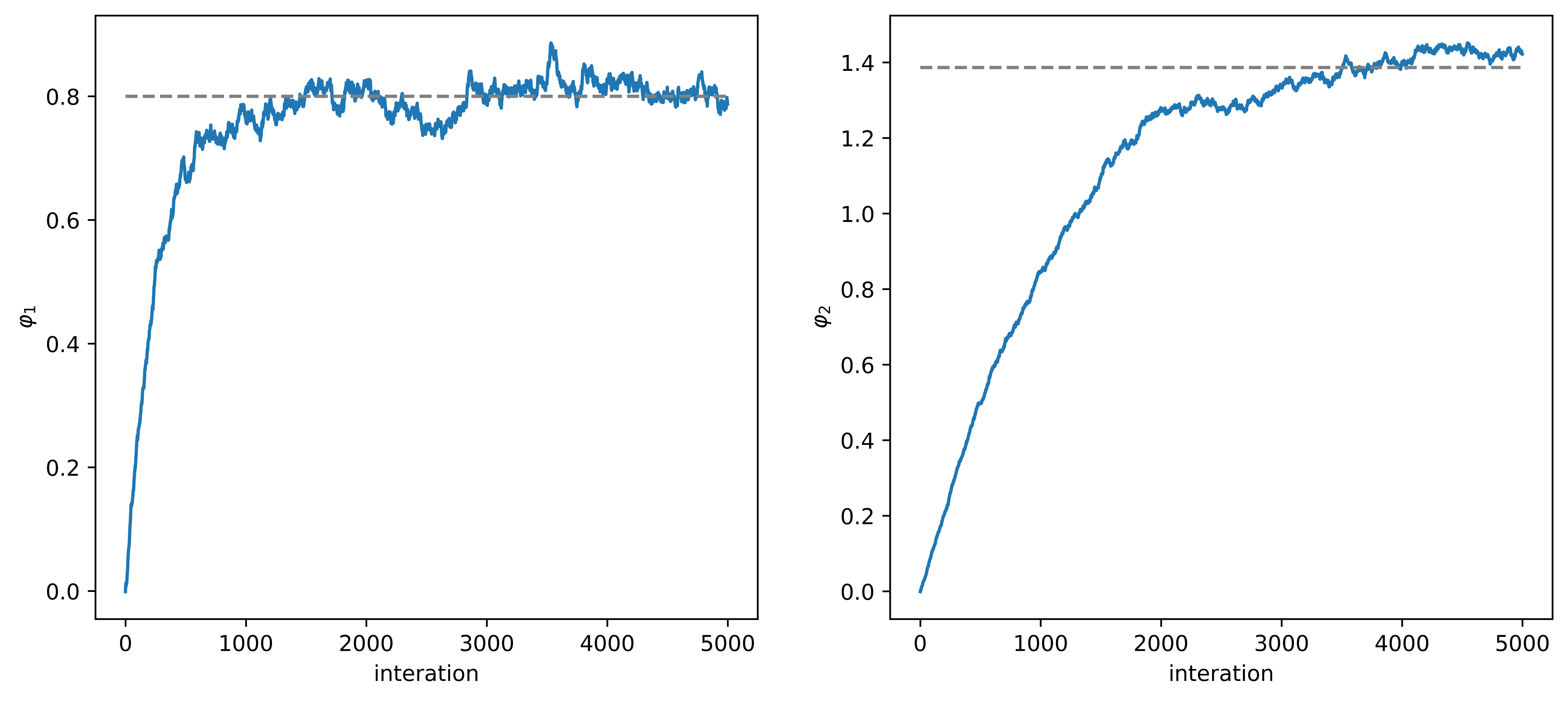}
	\caption{dynamics of $\varphi$}\label{fig:varphi}
\end{figure}
As shown in the figure, $\varphi_1$ and $\varphi_2$ converge to the true values quickly. Especially for $\varphi_1$, it converges after 1000 iterations. It is worth mentioning that the entire experiment took less than half a minute. Next, we examine the approximation of $f^\theta$ to $f$. We conduct five experiments, with the results shown in Figure \ref{fig:f}.
\begin{figure}[htbp]
	\centering
	\includegraphics[width=0.8\linewidth]{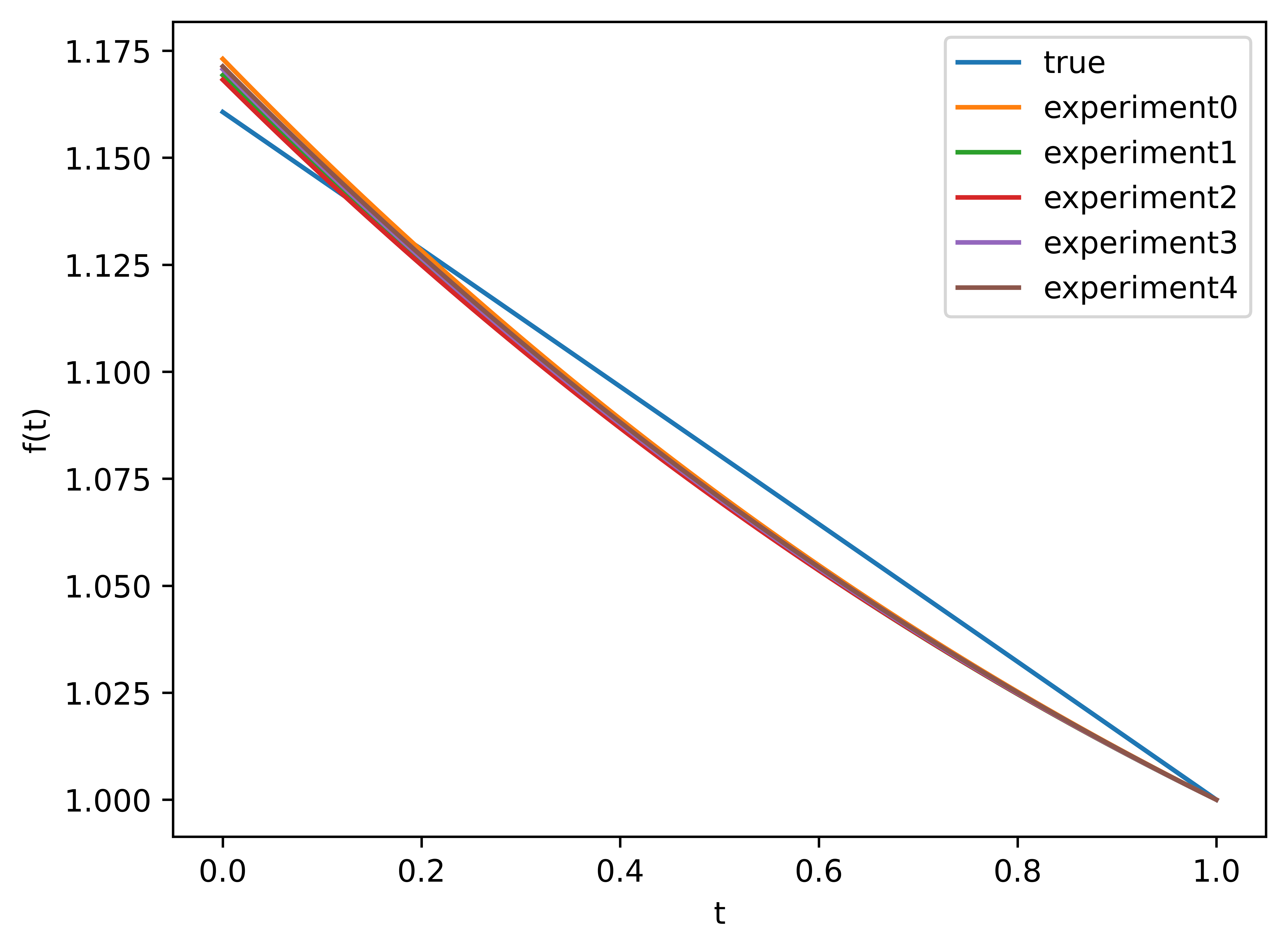}
	\caption{approximation of $f$}\label{fig:f}
\end{figure}
We can observe that even though we are using the polynomial approximation instead of a neural network, the approximation performance is quite satisfactory. \\
\indent	Secondly, to assess the robustness of the algorithm, we conduct simulations with different configurations:$\mu \in \{\pm0.1,\pm0.2,\pm0.3,\pm0.4\}$, $\sigma \in \{0.3,0.4,0.5\}$, while keeping other hyper-parameters unchanged. The numerical results are reported in Table \ref{table:diff}. We observe that regardless of parameter changes, the learning results are very close to the true values. Furthermore, when the true value of $|\varphi_1|$ is large, the error tends to increase. This is because we initialize the training parameter to 0, which requires more iterations to converge to larger true values.\\
\indent Finally, we investigate the impact of $\gamma$ on exploration performance. For different values of $\gamma$, we use the same training set for training to facilitate comparison. Performance under different $\gamma$ is shown in Figure \ref{fig:lam}. As $\gamma$ increases, the learned value of $\varphi_1$ converges more quickly to the true value, but eventually it tends to overestimate the true value. For $\varphi_2$, when iterations end, the learned value has not yet converged for small $\gamma$. This suggests that we should carefully consider the exploration weight.

\begin{figure}[htbp]
	\centering
	\includegraphics[width=1\linewidth]{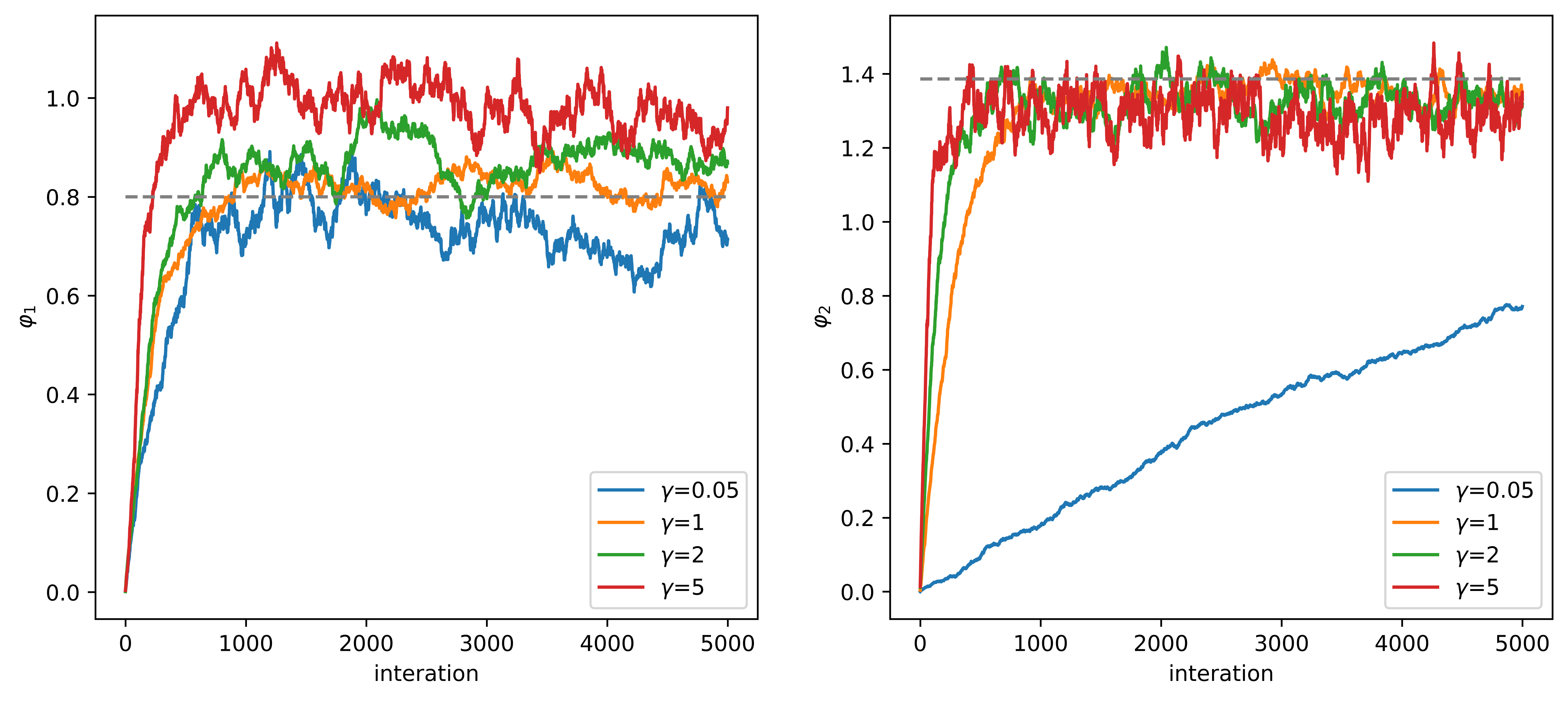}
	\caption{Performance under different $\gamma$ }\label{fig:lam}
\end{figure}

\begin{table}[ht!]
	\centering
	\caption{Convergence results under different $\mu,\sigma$}\label{table:diff}
	\begin{tabular}{cc|cc|cc}
		\toprule
		\multicolumn{1}{c}{\multirow{2}{*}{$\mu$}}&\multicolumn{1}{c|}{\multirow{2}{*}{$\sigma$}}&\multicolumn{2}{c|}{$\varphi_1$}&\multicolumn{2}{c}{$\varphi_2$}\\ \cline{3-6}
		& & True value & Learned value& True value& Learned value \\ \hline
		0.1 & 0.3 & 1.111 & 1.089 & 2.408 & 2.371 \\
	0.2 & 0.3 & 2.222 & 2.202 & 2.408 & 2.425 \\
	0.3 & 0.3 & 3.333 & 3.316 & 2.408 & 2.388 \\
	0.4 & 0.3 & 4.444 & 4.271 & 2.408 & 2.328 \\
	
	-0.1 & 0.3 & -1.111 & -1.085 & 2.408 & 2.353 \\
	-0.2 & 0.3 & -2.222 & -2.293 & 2.408 & 2.488 \\
	-0.3 & 0.3 & -3.333 & -3.326 & 2.408 & 2.327 \\
	-0.4 & 0.3 & -4.444 & -4.239 & 2.408 & 2.278 \\
	
	0.1 & 0.4 & 0.625 & 0.620 & 1.833 & 1.842 \\
	0.2 & 0.4 & 1.25 & 1.228 & 1.833 & 1.800 \\
	0.3 & 0.4 & 1.875 & 1.850 & 1.833 & 1.827 \\
	0.4 & 0.4 & 2.5 & 2.524 & 1.833 & 1.838 \\
	
	-0.1 & 0.4 & -0.625 & -0.645 & 1.833 & 1.815 \\
	-0.2 & 0.4 & -1.25 & -1.226 & 1.833 & 1.800 \\
	-0.3 & 0.4 & -1.875 & -1.838 & 1.833 & 1.824 \\
	-0.4 & 0.4 & -2.5 & -2.349 & 1.833 & 1.709 \\
	
	0.1 & 0.5 & 0.4 & 0.452 & 1.386 & 1.382 \\
	0.2 & 0.5 & 0.8 & 0.823 & 1.386 & 1.314 \\
	0.3 & 0.5 & 1.2 & 1.183 & 1.386 & 1.362 \\
	0.4 & 0.5 & 1.6 & 1.588 & 1.386 & 1.477 \\
	
	-0.1 & 0.5 & -0.4 & -0.383 & 1.386 & 1.358 \\
	-0.2 & 0.5 & -0.8 & -0.772 & 1.386 & 1.339 \\
	-0.3 & 0.5 & -1.2 & -1.240 & 1.386 & 1.367 \\
	-0.4 & 0.5 & -1.6 & -1.503 & 1.386 & 1.310 \\

		\bottomrule
	\end{tabular}
\end{table}

	\section{Conclusion}\label{sec:conclusions}
In this paper, we study the utility maximization problem with the CRRA utility function under the reinforcement learning framework, where the Tsallis entropy is introduced to encourage exploration. In the classical sense, the Merton problem is always well-posed and has a closed-form solution, regardless of how the index $p$ of the CRRA utility function is chosen. However, after introducing reinforcement learning, the problem becomes more complex. We first show that time-dependent exploration functions lead to the ill-posedness of the control problem, and are inspired to choose an appropriate primary temperature function such that a (semi-)closed-form solution becomes possible. We derive the exploratory HJB equation and solve for the optimal distributional strategy. However, a (semi-)closed-form strategy is not possible for all parameters.  When $0<p<1$ and $\beta=1$ or $3$,  the control problem is always well-posed and has a closed-form solution. When $\beta=1$ the optimal strategy distribution is a Gaussian distribution, while when $\beta=3$ the optimal strategy is a semicircular distribution defined on a compact support. This demonstrates that Tsallis entropy excels in scenarios with prevalent non-Gaussian, heavy-tailed behavior on compact support. The results of multiple assets are deferred in Appendix \ref{appen7.6}. When $p<0$, certain combinations of parameters may lead to the non-existence of an optimal strategy. Additionally, we also study the convergence of value function and strategy in the exploratory control problem to that in the classical stochastic control problem as the level of exploration approaches zero. Finally, we design a reinforcement learning algorithm and conduct numerical experiments. Although we obtain a semi-analytic value function, the approximation using polynomial basis functions yields very good results. We also study the impact of the exploration level on the convergence of trained parameters.\\
\indent There are still some open questions. First of all, is it possible to provide a unified answer to the well-posedness of such exploratory control problems, rather than relying on specific cases? Second, the parameterization of the strategy and the value function is closely tied to their analytical properties. Is it possible to set up a reinforcement learning algorithm without any prior information about the strategy and value function? We will study these issues in the future.
\section{Appendix}

\subsection{Derivation of Exploratory Wealth Process}\label{subsec:derivation}
In discrete-time setting, we have
\begin{equation*}
	\Delta W_t= [r+(\mu-r)u_t]W_t \Delta t +\sigma u_t W_t \Delta B_t.
\end{equation*}
Since $u_t$ is sampled from $\pi_t$ which is independent of the Brownian motion of the market, by the law of large number  we  obtain the first and second moments of the increment of wealth:
\begin{equation}\label{moment}
	\begin{split}
		E[\Delta W_t]=\left[r+(\mu-r)\int_{\mathbb{R}}u\pi_t(u)du\right]W_t \Delta t,\\
		E[(\Delta W_t)^2]=\left[\sigma^2\int_Ru^2\pi_t(u)du\right]W_t^2 \Delta t+o(\Delta t).\\
	\end{split}
\end{equation}
where $\lim\limits_{t\to 0}o(\Delta t)/t=0$. Inspired by (\ref{moment}), we derive the exploratory state process:

\begin{equation*}
		dW_t^\pi = \Big[r+(\mu -r)\int_{\mathbb{R}}u\pi_t(u)du\Big]W_t^\pi dt+ \sigma W_t^\pi\sqrt{\int_{\mathbb{R}}u^2\pi_t(u)du}dB_t.
\end{equation*}
 It is easy to check that the first and second moments of the exploratory process satisfy (\ref{moment}).

\subsection{Proof of Proposition \ref{them}}\label{pf_them}
\begin{proof}
	We only prove the case where $0<p<1$ (b<0), leaving the proof of the other case to the readers. We regard $h$ as a function of $(t,y)$ defined on $[0,+\infty)\times (0,+\infty)$. Note that $h\in C^1((0,+\infty)\times (0,+\infty))$, which implies $h$ is locally Lipschitz continuous on its domain. According to the Picard theorem and the extendability of solutions (cf. Kong  \cite{kong2014short}), the ODE (\ref{eq:ODE2}) has a unique solution through the initial point $(t,y)=(0,1)$ and it can be extended continuously up to the boundary of $(0,\infty)\times(0,\infty)$ (in the positive time direction). Therefore, the maximal existence interval of the solution the ODE (\ref{eq:ODE2}) can be analyzed through studying the properties of $h$. According to Lemma \ref{lemma}, we consider all possible cases as follows.
	\begin{enumerate}
		\item If $a\leq 0$, $h(\cdot)$ is strictly decreasing and there exists an unique zero denoted by $y_0$. \\
		(i) If $y_0=1$, then the unique solution is $y(t)\equiv 1$ for $t \geq 0$. \\
		(ii)If $y_0<1$, the solution will decrease to $y_0$ but never equal $y_0$. Otherwise, if there exists $t_0>0$ such that $y(t_0)=y_0$, then $y(\cdot)\equiv y_0$, which contradicts the initial condition.Thus, in this case, the maximal existence interval of $y(\cdot)$ is $[0,+\infty)$.\\
		(iii)If $y_0>1$, the solution will increase to $y_0$ but never equal $y_0$. In this case, the maximal existence interval of $y(\cdot)$ is $[0,+\infty)$.
		\item  If $a>0$ and $b[\log(- b)-\log(a)-1]+c>0$, then $h(y)>0$ for any $y>0$. The solution $y(t)$ strictly increases to $+\infty$. Note that $ h(y)\leq Cy+C$ when $y\geq 1$, where $C>0$ represents a generic constant which may differ at different places in the proof. By comparison theorem of ODE, we have $y(t)\leq Cte^{Ct}+C$. Therefore, the maximal existence interval of $y(\cdot)$ is $[0,+\infty)$.
		\item   If $a>0$ and $b[\log(- b)-\log(a)-1]+c=0$, there exists an unique zero of $h$ denote by $y_0$ and $h(y)\geq 0$.\\
		(i) If $y_0=1$, then $y(\cdot)\equiv1$.\\
		(ii) If $1\in (0,y_0)$, then the solution will increase to $y_0$ but never equal $y_0$. \\
		(iii) If $1\in (y_0,+\infty)$, then the solution will increase to $+\infty$. The maximal existence interval is $[0,+\infty)$ according to the argument in case 2.
		\item  If $a>0$ and $b[\log(- b)-\log(a)-1]+c<0$, there exists two distinct zeros of $h$ denoted by $y_1,y_2$ and $0<y_1<-b/a<y_2<\infty$.\\
		(i) If $y_i=1$, $i=1$ or 2, then $y(\cdot)\equiv 1$.\\
		(ii) If $1\in(0,y_1)$, then $y(\cdot)$ will increase to $y_1$ but never equal $y_1$.\\
		(iii) If $1\in(y_1,y_2)$, then $y(\cdot)$ will decrease to $y_1$ but never equal $y_1$.\\
		(iv) If $1\in(y_2,+\infty)$, then the solution will increase to $+\infty$ and  the maximal existence interval is $[0,+\infty)$.
	\end{enumerate}
\end{proof}

\subsection{Proof of Proposition \ref{coro_wellpose}}\label{proof_coro_wellpose}
The results in Case 1, 2.1 and 2.2 are obvious. We focus on Case 2.3.
By principle of dynamic programming, for $(t,w)\in[0,T-\delta)\times \mathbb{R}_+$ and any admissible control $\pi$, we have
\begin{equation*}
	\begin{split}
		V(t,w)&\geq E\Big\{\int_{t}^{T-\delta}\Big[\gamma( W_s^\pi)^p \int_{\mathbb{R}}\pi_s(u)	\log\pi_s(u)du\Big]ds+V(T-\delta,W_{T-\delta}^\pi)\Big\}\\
		&=  E\Big\{\int_{t}^{T-\delta}\Big[\gamma( W_s^\pi)^p \int_{\mathbb{R}}\pi_s(u)	\log\pi_s(u)du\Big]ds\Big\}
	\end{split}
\end{equation*}
where the last equality is due to $f(T-\delta)=y(\delta)=0$, then $V(T-\delta,w)=0$. Consider the following admissible controls
\begin{equation*}
	\pi^n_s(u)\sim \mathcal{N}(0,n^2),\quad t\leq s\leq T-\delta.
\end{equation*}
Then 
\begin{equation*}
	\begin{split}
		&E\Big\{\int_{t}^{T-\delta}\Big[\gamma( W_s^{\pi^n})^p \int_{\mathbb{R}}\pi_s^n(u)	\log\pi_s^n(u)du\Big]ds\Big\}\\
		=&\frac{\gamma w^p}{2}\log(2\pi e n^2)\int_{t}^{T-\delta}\exp\Big\{pr( s-t)+\Big(\frac{p^2-p}{2}\Big)\sigma^2n^2(s-t)\Big\}ds.
	\end{split}
\end{equation*}
After taking the limit with respect to $n$, we conclude that $V(t,w)=\infty$ for $(t,w)\in [0,T-\delta)\times \mathbb{R}_+$.
\subsection{Proof of Theorem \ref{ver_them}}\label{proof_ver_them}
We only prove Case 1, and the proof for the other cases are similar. It is easy to check $v\in C^{1,2}([0,T)\times R_{+})$. For all $(t,w)\in [0,T)\times R_{+}$ and $\pi\in \Pi(t,w)$, we have $W_s^\pi>0$ for $t\leq s\leq T$. Define $\tau_n=\inf\{s\geq t:\int_{t}^{s}(W_r^\pi)^{p}dr\geq n\}$. Notice that $\tau_n\to \infty$ when $n$ goes to infinity. Applying It\^o's formula to $v$ between $t$ and $T\wedge \tau_n$ gives us 
\begin{equation}
	\begin{split}
		v(T\wedge\tau_n,W_{T\wedge\tau_n}^\pi)=&v(t,w)+\int_{t}^{T\wedge\tau_n}v_t(s,W_s^\pi)+\mathcal{L}^\pi v(s,W_s^\pi)ds\\
		&+\int_{t}^{T\wedge \tau_n}\Bigg[\sigma W_s^\pi\sqrt{\int_{\mathbb{R}}u^2\pi_s(u)du}\Bigg]v_wdB_s,
	\end{split}
\end{equation}
where 
$$\mathcal{L}^\pi v(s,W_s^\pi)=\Big\{\Big[r+(\mu -r)\int_{\mathbb{R}}u\pi_s(u)du\Big]W_s^\pi \Big\}v_w+\Big\{\frac{1}{2}\sigma^2(W_s^\pi)^2\int_{\mathbb{R}}u^2\pi_s(u)du\Big\}v_{ww}. $$
Taking the expectation, we get
$$\mathbb{E} [v(T\wedge\tau_n,W_{T\wedge\tau_n}^\pi)\mid W^\pi_t=w]=v(t,w)+\mathbb{E}\Big[\int_{t}^{T\wedge\tau_n}v_t(s,W_s^\pi)+\mathcal{L}^\pi v(s,W_s^\pi)ds\mid W^\pi_t=w\Big]$$
since the diffusion term is a martingale. Noting that $v$ satisfies exploratory HJB equation, we have
$$v_t(s,W_s^\pi)+\mathcal{L}^\pi v(s,W_s^\pi) \leq \gamma(W_s^\pi)^p \int_{\mathbb{R}}\pi_s(u)du\log \pi_s(u)$$
and so
\begin{eqnarray}\label{eq:ver}
	&\mathbb{E} [v(T\wedge\tau_n,W_{T\wedge\tau_n}^\pi)\mid W^\pi_t=w]\\
    \leq &v(t,w)+\mathbb{E}[\int_{t}^{T\wedge \tau_n}[\gamma(W_s^\pi)^p \int_{\mathbb{R}}\pi_s(u)\log \pi_s(u)du]ds \mid W^\pi_t=w]. \notag
\end{eqnarray}
Note that
$$\Bigg|\int_{t}^{T\wedge \tau_n}\Big[\gamma(W_s^\pi)^p \int_{\mathbb{R}}\pi_s(u)\log \pi_s(u)du\Big]ds\Bigg| \leq \int_{t}^{T}\Bigg|\gamma(W_s^\pi)^p \int_{\mathbb{R}}\pi_s(u)\log \pi_s(u)du\Bigg|ds.$$
The right-hand-side term is integrable by the integrability condition on $\Pi(t,w)$. Then by Fatou's lemma and dominated convergence theorem we can obtain
\begin{equation*}
	v(t,w) \geq J(t,w; \pi).
\end{equation*}
It implies $v(t,w)\geq V(t,w)$. On the other hand, for $\pi=\hat{\pi}$, equality holds in (\ref{eq:ver}) and hence $v(t,w)=V(t,w)$ and $\hat{\pi}$ is optimal strategy.

\subsection{Proof of Lemma \ref{pro_conv}}\label{proof_pro_conv}

We first prove that there exists $N>0$ such that $\min\limits_{0\leq t \leq T} y^\gamma(t) \geq N$ for any $\gamma>0$.  Fix $y>0$. By direct calculation, it is easy to check that 
	$$-\frac{p\gamma}{2}\log(y)-\frac{p\gamma}{2}\log\Big(\frac{\sigma^2(1-p)}{2\pi\gamma}\Big) \geq -Ky$$
	where $K$ is a positive constant. Thus $h^\gamma(y)\geq (a-K)y$. The solution of ODE $y'=(a-K)y$ with initial condition $y(0)=1$ is given by $\tilde{y}(t)=\exp\Big((a-K)t\Big)$.
	According to comparison theorem, we have 
	$y^\gamma(t)\geq \tilde{y}(t)$ for any $\gamma>0$. Besides, since $N=\min\limits_{0\leq t\leq T}\tilde{y}(t)>0$, we conclude that $\min\limits_{0 \leq t\leq T}y^\gamma (t)\geq N$ for any $\gamma>0$.

  We proceed to show that $y^\gamma$ has a convergent subsequence in $C([0,T])$.
 Note that $y^\gamma \geq N>0$ in $[0,T]$, we obtain
	\begin{equation*}
		\begin{dcases}
			|h^\gamma(N)| \leq M,\\
			|h^\gamma(y_1)-h^\gamma(y_2)| \leq L|y_1-y_2|,\quad \forall y_1,y_2\in [N,+\infty),
		\end{dcases}
	\end{equation*}
	where $M,L$ are positive constants independent of $\gamma$ provided $\gamma$ is small. Therefore,
	\begin{equation*}
		\begin{split}
			|y^\gamma(t)|& \leq 1+\int_{0}^{t}h^\gamma(y^\gamma(s))ds\\& \leq 1+\int_{0}^{t}\Big[|h^\gamma(y^\gamma(s))-h^\gamma(N) |+|h^\gamma(N)|\Big]ds\\
			& \leq 1+MT+LNT+\int_{0}^{t}L|y^\gamma(s)|ds.
		\end{split}
	\end{equation*}
	By Gronwall inequality, we have
	$$|y^\gamma(t)|\leq (1+MT+LNT)e^{LT}<\infty,~\forall t\in[0,T].$$
	It implies that $y^\gamma(\cdot)$ is uniformly bounded. On the other hand,
	$$|y^\gamma(t)-y^\gamma(s)| \leq \int_{s}^{t}L|y^\gamma(r)|dr\leq (1+MT+LNT)e^{LT}L(t-s).$$
Thus $y^\gamma(\cdot)$ is equicontinuous. By Arzelà-Ascoli Theorem, there exists $\xi(\cdot)\in C([0,T])$ such that a subsequence of $y^\gamma(\cdot)$ which is assumed to be the sequence itself for simplicity, satisfies
$$\lim\limits_{\gamma \to 0^+}\max_{t\in [0,T]}|y^\gamma(t)-\xi(t)|=0.$$
It remains to prove $\xi(\cdot)\equiv y^0(\cdot)$. Note that
\begin{equation*}
	y^\gamma(t) =1+\int_{0}^{t}h^\gamma(y^\gamma(s))ds.
\end{equation*}
Taking the limit with respect to $\gamma$, by Dominated Convergence Theorem, we obtain
$$\xi(t)=1+\int_{0}^{t}h^0(\xi(s))ds.$$
It follows that $\xi(t)=y^0(t)$ for any $t\in[0,T]$.

\subsection{Multiple Assets}\label{appen7.6}
There are no significant technical challenges in extending our method to the multi-asset case; however, the computations become more cumbersome. We consider the case with $d$ risky assets and a bond that pays a constant risk-free rate $r$.  Let $\bm{S}_t$ be the price vector of $d$ risky assets at time $t$, i.e., $\bm{S}_t=(S_t^1,S_t^2,\cdots,S_t^d)^T$. It satisfies the following equation:
$$d\bm{S}_t=\text{Diag}(\bm{S}_t)(\bm{\mu} dt+\bm{\sigma} dB_t),$$
where $B_t \in \mathbb{R}^d$ is a standard $d-$dimensional Brownian motion, $\bm{\mu} \in \mathbb{R}^d$ is a drift vector, $\bm{\sigma} \in \mathbb{R}^{d\times d}$ is a non-degenerate volatility matrix. Denote the fractions of the total wealth invested in the stocks at time $t$ by $\bm{u}_t\in \mathbb{R}^d$. Thus a self-financing wealth process $W_t$ can be descried by
\begin{equation*}\label{mult_cl_wealth}
	dW_t = W_t[r+(\bm{\mu}-\bm{r})^{T}\bm{u}_t] dt + W_t \bm{u}_t^T \bm{\sigma}  dB_t, ~W_0=w_0>0.
\end{equation*}
where we use $\bm{r}$ to represent a vector whose components are all equal to $r$. To randomize the classical strategy, we consider the feedback exploratory policies $\pi:(t,w)\in [0,T]\times \mathbb{R}^d \mapsto \pi(\cdot|t,w)\in \mathcal{P}(\mathbb{R}^d)$, where $\mathcal{P}(\mathbb{R}^d)$ stands for all absolutely continuous probability density functions in $\mathbb{R}^d$. By the similar derivation as above, the exploratory wealth process corresponding to multiple assets is give by
$$dW_t^\pi=W_t^\pi\Big[r+\int_{\mathbb{R}^d}(\bm{\mu}-\bm{r})^{T}\bm{u}_t\pi_t(\bm{u})d\bm{u}\Big]dt + W_t^\pi\sqrt{\int_{\mathbb{R}^d}\bm{u}_t^T\bm{\sigma}\bm{\sigma}^T\bm{u}_t\pi_t(\bm{u})d\bm{u}}  dB_t.$$
Subsequently, the exploratory control problem can be defined in the same manner as in the case with a single risky asset.  Using similar method as above, we derive the exploratory HJB equation \footnote{For simplicity, we only consider the case where $0<p<1$ and $\lambda(t,w)=\gamma w^p$. }:
\begin{equation*}\label{tsallis_hjb_multiple}
	\begin{aligned}0=v_{t}+rw v_{w}+\sup\limits_{\pi}\Big\{\int_{\mathbb{R}^d}&\Big[wv_{w}(\bm{\mu}-\bm{r})^T\bm{u}+\frac{1}{2}w^{2}v_{ww}\bm{u}^T\bm{\sigma}\bm{\sigma}^T\bm{u}\Big]\pi(\bm{u})d\bm{u}\\&+\gamma w^p\int_{\mathbb{R}^d}H_\beta(\pi(\bm{u}))d\bm{u}\Big\}=0,\end{aligned}
\end{equation*}
with terminal condition $v(T,w)=w^p/p$. By solving it, we can derive optimal exploratory policies. We summarize the results as follows.
\begin{proposition}
	When $\beta=0$, the optimal policy is given by
	\begin{equation*}
		\hat{\pi}(u|t,w)\sim\mathcal{N}\Bigg(\frac{(\bm{\sigma}\bm{\sigma}^T)^{-1}(\bm{\mu}-\bm{r})}{1-p},\frac{\gamma}{(1-p)f(t)}(\bm{\sigma}\bm{\sigma}^T)^{-1}\Bigg),
	\end{equation*}
	where $f(t)\in C^1([0,T])$ is an unique solution of the following ODE:
	\begin{equation*}\label{eq:ODE_multiple1}
		\begin{dcases}
			f'(t)=\Big[-\frac{p (\bm{\mu}-\bm{r})^T(\bm{\sigma}\bm{\sigma}^T)^{-1}(\bm{\mu}-\bm{r})}{2(1-p)}-rp\Big]f(t)+\frac{p \gamma}{2} \log(f(t))+\frac{p \gamma}{2}\log\Big(\frac{|\bm{\sigma}\bm{\sigma}^T|(1-p)}{(2 \pi)^d \gamma}\Big),\\ 
			f(T)=1,
		\end{dcases}
	\end{equation*}
	where $|\cdot|$ denotes the determinant of matrix.\\
	\indent	When $\beta=3$, the optimal policy is given by
	\begin{equation*}
		\hat{\pi}(u|t,w)=\sqrt{\frac{(1-p)f(t)}{3\gamma}}\sqrt{R^2(t)-\Big(\bm{u}-\bm{\theta}\Big)^T(\bm{\sigma}\bm{\sigma}^T)\Big(\bm{u}-\bm{\theta}\Big)},
	\end{equation*}
	where 
	\begin{equation*}
		\begin{dcases}
			\bm{\theta} = \frac{(\bm{\sigma}\bm{\sigma}^T)^{-1}(\bm{\mu}-\bm{r})}{1-p},\\
			R(t) = \pi^{-\frac{1}{2}}\Big[2\Gamma\Big(1+\frac{d+1}{2}\Big)\Big]^{\frac{1}{d+1}}\Big[\frac{3\gamma |\bm{\sigma}|^2}{(1-p)f(t)}\Big]^{\frac{1}{2(d+1)}},
		\end{dcases}
	\end{equation*}
	where $\Gamma(\cdot)$ is the Gamma function. And $f(t)\in C^1([0,T])$ is an unique solution of the following ODE:
	\begin{equation*}\label{eq:ODE_multiple2}
		\begin{dcases}
			f'(t)=\Big[-\frac{p (\bm{\mu}-\bm{r})^T(\bm{\sigma}\bm{\sigma}^T)^{-1}(\bm{\mu}-\bm{r})}{2(1-p)}-rp\Big]f(t)+pK[f(t)]^{\frac{d}{d+1}}-\frac{p \gamma}{2},\\ 
			f(T)=1,
		\end{dcases}
	\end{equation*}
	where
\begin{equation*}
	\begin{split}
		K=\frac{d(1-p)}{2\pi(d+3)}\Big(\frac{3\gamma |\bm{\sigma}|^2}{(1-p)}\Big)^{\frac{1}{d+1}}\Big[2\Gamma\Big(\frac{d+3}{2}\Big)\Big]^{\frac{2}{d+1}}\Big[\frac{\Gamma(d/2)}{\Gamma((d+1)/2)}\Big].
	\end{split}
\end{equation*} 

\end{proposition}


\begin{thebibliography}{HD}
\bibitem{bo2024continuous} L. Bo, Y. Huang, X. Yu, and T. Zhang,
Continuous-time q-learning for jump-diffusion models under Tsallis entropy,
\emph{arXiv preprint arXiv:2407.03888} (2024).

\bibitem{boyce2017elementary} W.E. Boyce, R.C. DiPrima, and D.B. Meade,
Elementary differential equations,
\emph{John Wiley \& Sons} (2017).

\bibitem{dai2023learning} M. Dai, Y. Dong, Y. Jia, and X.Y. Zhou,
Learning Merton's Strategies in an Incomplete Market: Recursive Entropy Regularization and Biased Gaussian Exploration,
\emph{arXiv preprint arXiv:2312.11797} (2023).

\bibitem{dai2024learning} M. Dai, Y. Sun, Z.Q. Xu, and X.Y. Zhou,
Learning to Optimally Stop a Diffusion Process,
\emph{Available at SSRN} (2024).

\bibitem{dai2009finite} M. Dai and F. Yi,
Finite-horizon optimal investment with transaction costs: A parabolic double obstacle problem,
\emph{Journal of Differential Equations} {246} (2009), 1445--1469.

\bibitem{dong2024randomized} Y. Dong,
Randomized optimal stopping problem in continuous time and reinforcement learning algorithm,
\emph{SIAM Journal on Control and Optimization} {62} (2024), 1590--1614.

\bibitem{donnelly2024exploratory} R. Donnelly and S. Jaimungal,
Exploratory control with Tsallis entropy for latent factor models,
\emph{SIAM Journal on Financial Mathematics} {15} (2024), 26--53.

\bibitem{duffie1992stochastic} D. Duffie and L.G. Epstein,
Stochastic differential utility,
\emph{Econometrica: Journal of the Econometric Society} (1992), 353--394.

\bibitem{elie2008optimal} R. Elie and N. Touzi,
Optimal lifetime consumption and investment under a drawdown constraint,
\emph{Finance and Stochastics} {12} (2008), 299--330.

\bibitem{epstein2013substitution} L.G. Epstein and S.E. Zin,
Substitution, risk aversion and the temporal behavior of consumption and asset returns: A theoretical framework,
\emph{Handbook of the fundamentals of financial decision making: Part i} (2013), 207--239.

\bibitem{guo2023exploratory} J. Guo, X. Han, and H. Wang,
Exploratory mean-variance portfolio selection with Choquet regularizers,
\emph{arXiv preprint arXiv:2307.03026} (2023).

\bibitem{han2023choquet} X. Han, R. Wang, and X.Y. Zhou,
Choquet regularization for continuous-time reinforcement learning,
\emph{SIAM Journal on Control and Optimization} {61} (2023), 2777--2801.

\bibitem{jia2022policy} Y. Jia and X.Y. Zhou,
Policy evaluation and temporal-difference learning in continuous time and space: A martingale approach,
\emph{Journal of Machine Learning Research} {23} (2022), 1--55.

\bibitem{jia2022policy_a} Y. Jia and X.Y. Zhou,
Policy gradient and actor-critic learning in continuous time and space: Theory and algorithms,
\emph{Journal of Machine Learning Research} {23} (2022), 1--50.

\bibitem{jia2023q} Y. Jia and X.Y. Zhou,
q-Learning in continuous time,
\emph{Journal of Machine Learning Research} {24} (2023), 1--61.

\bibitem{jiang2022reinforcement} R. Jiang, D. Saunders, and C. Weng,
The reinforcement learning Kelly strategy,
\emph{Quantitative Finance} {22} (2022), 1445--1464.

\bibitem{kong2014short} Q. Kong,
A short course in ordinary differential equations,
\emph{Springer} (2014).

\bibitem{lim2011optimal} B.H. Lim and Y.H. Shin,
Optimal investment, consumption and retirement decision with disutility and borrowing constraints,
\emph{Quantitative Finance} {11} (2011), 1581--1592.

\bibitem{liu2002optimal} H. Liu and M. Loewenstein,
Optimal portfolio selection with transaction costs and finite horizons,
\emph{The Review of Financial Studies} {15} (2002), 805--835.

\bibitem{merton1975optimum} R.C. Merton,
Optimum consumption and portfolio rules in a continuous-time model,
\emph{Stochastic optimization models in finance} (1975), 621--661.

\bibitem{shreve1994optimal} S.E. Shreve and H.M. Soner,
Optimal investment and consumption with transaction costs,
\emph{The Annals of Applied Probability} (1994), 609--692.

\bibitem{sutton2018reinforcement} R.S. Sutton,
Reinforcement learning: An introduction,
\emph{A Bradford Book} (2018).

\bibitem{szepesvari2022algorithms} C. Szepesvári,
Algorithms for reinforcement learning,
\emph{Springer Nature} (2022).

\bibitem{tang2022exploratory} W. Tang, Y.P. Zhang, and X.Y. Zhou,
Exploratory HJB equations and their convergence,
\emph{SIAM Journal on Control and Optimization} {60} (2022), 3191--3216.

\bibitem{tsallis1988possible} C. Tsallis,
Possible generalization of Boltzmann-Gibbs statistics,
\emph{Journal of Statistical Physics} {52} (1988), 479--487.

\bibitem{wang2020reinforcement} H. Wang, T. Zariphopoulou, and X.Y. Zhou,
Reinforcement learning in continuous time and space: A stochastic control approach,
\emph{Journal of Machine Learning Research} {21} (2020), 1--34.

\bibitem{wang2020continuous} H. Wang and X.Y. Zhou,
Continuous-time mean--variance portfolio selection: A reinforcement learning framework,
\emph{Mathematical Finance} {30} (2020), 1273--1308.

\bibitem{wu2024reinforcement} B. Wu and L. Li,
Reinforcement learning for continuous-time mean-variance portfolio selection in a regime-switching market,
\emph{Journal of Economic Dynamics and Control} {158} (2024), 104787.

\end{thebibliography}
\end{document}